\documentclass[journal]{IEEEtran}
\usepackage{algorithmic}
\usepackage[ruled,linesnumbered]{algorithm2e}
\usepackage{array}
\usepackage[caption=false]{subfig}
\usepackage{textcomp}
\usepackage{stfloats}
\usepackage{url}
\usepackage{verbatim}
\usepackage{graphicx}
\usepackage{cite}
\usepackage[english]{babel}
\usepackage{latexsym,amssymb,amsfonts,amsbsy}
\usepackage{amsthm}
\usepackage{amsmath,mathrsfs}
\usepackage{graphicx}
\usepackage{indentfirst}
\usepackage{float}
\usepackage{verbatim}
\usepackage{bm}
\usepackage{setspace}
\usepackage[colorlinks=true, allcolors=green]{hyperref}
\usepackage{geometry}
\usepackage{colortbl}
\usepackage{multirow}
\usepackage{booktabs}
\usepackage{bigstrut}
\usepackage{overpic}

\def\vt{{\bf t}}

\def\cJ{\mathcal{J}}
\usepackage{array}
\newcolumntype{P}[1]{>{\centering\arraybackslash}p{#1}}
\newcolumntype{M}[1]{>{\centering\arraybackslash}m{#1}}

\newcommand\norm[1]{\left\Vert#1\right\Vert}

\def\argmin{\mathop{\rm argmin}}

\def\R{\mathbb{R}}

\def\rank{\mathrm{rank}}

\def\bU{\boldsymbol{U}}
\def\bV{\boldsymbol{V}}
\def\bI{\boldsymbol{I}}

\def\bW{\boldsymbol{W}}
\def\bS{\boldsymbol{S}}

\def\bY{\boldsymbol{Y}}
\def\bX{\boldsymbol{X}}

\def\bE{\boldsymbol{E}}
\def\bM{\boldsymbol{M}}
\newtheorem{theorem}{Theorem}
\newtheorem{lemma}{Lemma}

\begin{document}

%\tr{Weight}
% Robust Principal Component Analysis Based on Weighted Least Squares and Low-Rank Matrix Factorization
\title{Robust PCA Based on Adaptive Weighted Least Squares and Low-Rank Matrix Factorization\thanks{Submitted to the editors \today. This work is supported by the National Natural Science Foundation of China (Grant No. 12361089); the Scientific Research Fund Project of Yunnan Provincial Education Department (Grant No. 2024J0642); the Yunnan Fundamental Research Projects (Grant Nos. 202401AU070104, 202401AU070105); the Scientific Research Fund Project of Yunnan University of Finance and Economics (Grant No. 2024D38);  the Foundation of MOE-LCSM, School of Mathematics and Statistics, Hunan Normal University (Grant No. 202405)}}

\author{Kexin Li\thanks{Kexin Li, School of Statistics and Mathematics, Yunnan University of Finance and Economics, Kunming, Yunnan, China.}, You-wei Wen$^*$\thanks{You-wei Wen, Corresponding author. Key Laboratory of Computing and Stochastic Mathematics (LCSM), School of Mathematics and Statistics, Hunan Normal University, Changsha, Hunan, China. Email: {\tt wenyouwei@gmail.com}},
Xu Xiao\thanks{Xu Xiao, School of Mathematics and Statistics, Guangxi Normal University, Guilin, Guangxi, China}, Mingchao Zhao \thanks{Mingchao Zhao, School of Statistics and Mathematics, Yunnan University of Finance and Economics, Kunming, Yunnan, China.}
}
\maketitle

\begin{abstract} 
Robust Principal Component Analysis (RPCA) is a fundamental technique for decomposing data into low-rank and sparse components, which plays a critical role for applications such as image processing and anomaly detection. Traditional RPCA methods commonly use $\ell_1$ norm regularization to enforce sparsity, but this approach can introduce bias and result in suboptimal estimates, particularly in the presence of significant noise or outliers. Non-convex regularization methods have been proposed to mitigate these challenges, but they tend to be complex to optimize and sensitive to initial conditions, leading to potential instability in solutions.	
To overcome these challenges, in this paper, we propose a novel RPCA model that integrates adaptive weighted least squares (AWLS) and low-rank matrix factorization (LRMF). The model employs a {self-attention-inspired} mechanism in its weight update process, allowing the weight matrix to dynamically adjust and emphasize significant components during each iteration. By employing a weighted F-norm for the sparse component, our method effectively reduces bias while simplifying the computational process compared to traditional $\ell_1$-norm-based methods. We use an alternating minimization algorithm, where each subproblem has an explicit solution, thereby improving computational efficiency. Despite its simplicity, numerical experiments demonstrate that our method outperforms existing non-convex regularization approaches, offering superior performance and stability, as well as enhanced accuracy and robustness in practical applications.
\end{abstract}


\section{Introduction}
Principal Component Analysis (PCA) is a widely used dimensionality reduction technique that projects high-dimensional data onto a lower-dimensional subspace by maximizing the variance along the principal components \cite{abdi2010principal, jolliffe2016principal}. The core of PCA lies in performing Singular Value Decomposition (SVD) on the data matrix to capture its dominant components, which correspond to the directions of maximum variance. Despite its practicality, PCA assumes that the data is clean and follows a Gaussian distribution, making it highly sensitive to noise and outliers \cite{xu2010robust}. This sensitivity poses a significant challenge in real-world applications, such as image processing and computer vision, where large errors or outliers often contaminate data. 

To address these limitations, Robust  PCA (RPCA) has been proposed as a more flexible and resilient variant of PCA \cite{nie2020truncated,bouwmans2018applications}.  
RPCA decomposes the given data matrix $\bY \in \R^{m\times n}$ into a low-rank matrix $\bX$, representing the underlying data structure, and a sparse matrix $\bS$, which captures the noise or outliers, {\it i.e.}, 
\begin{equation}
	\bY = \bX +\bS. 
\end{equation}
 This approach effectively separates the signal from the noise, enhancing data analysis's robustness in practical applications. RPCA has proven particularly useful in areas such as image processing \cite{gu2017weighted, bouwmans2018applications,wang2021tensor,jiang2018superpca}, background subtraction \cite{liu2021efficient,cao2016total}, and anomaly detection \cite{ruhan2022enhance,yao2022hyperspectral,xu2018joint,xiao2023robust}, where the data is often heavily corrupted or contains sparse, high-magnitude outliers. 
 
 Intuitively, RPCA can be formulated as a problem of minimizing the rank of the low-rank component and the $\ell_0$ norm of the sparse component \cite{candes2011robust}. Mathematically, this translates to an optimization problem where one seeks to find a low-rank matrix that best represents the underlying structure of the data, while simultaneously minimizing the number of non-zero entries in the sparse matrix that captures outliers or noise. However, both the rank function and the $\ell_0$ norm are non-convex, making this optimization problem NP-hard. Solving it directly is computationally intractable, especially for large-scale data. To circumvent this challenge, common approaches replace the rank function with the nuclear norm and the $\ell_0$ norm with the $\ell_1$ norm \cite{recht2010guaranteed,cai2010singular}. The nuclear norm serves as the best convex surrogate for the rank function, promoting low-rank solutions, while the $\ell_1$ norm is a well-known convex relaxation of the $\ell_0$ norm, favouring sparsity. This relaxation transforms the original non-convex problem into a convex optimization problem, which can be solved efficiently by existing convex optimization techniques with well-established convergence such as the accelerated proximal gradient (APG) \cite{toh2010accelerated,li2015accelerated}, the alternating direction method (ALM) \cite{Yuan2013SparseAL},  the inexact augmented Lagrange multiplier (IALM) \cite{lin2011linearized}, and so on.  
 
 Nuclear norm-based methods impose uniform shrinkage on singular values, leading to biased estimates and suboptimal reconstructions. To address this, non-convex low-rank regularization techniques have gained attention\cite{kang2015robust,oh2015partial,zhang2023generalized,zhang2023hyperspectral}. Gu {\it {et al.}}  \cite{gu2017weighted} proposed the Weighted Nuclear Norm Minimization (WNNM) model, which reduces excessive penalization of large singular values by assigning different weights. Xie {\it {et al.}}  \cite{xie2016weighted} introduced the Weighted Schatten-$p$ Norm Minimization (WSNM) method, offering refined control over singular values using the Schatten-$p$ norm. Huang {\it {et al.}}  \cite{huang2023robust}  proposed a truncated $\ell_{1-2}$ norm approach, providing more accurate low-rank structure capture and noise separation. These methods outperform traditional nuclear norm-based approaches in various applications by preserving the data structure more effectively.

However, the main bottleneck of the methods above lies in the need to perform SVD at each iteration. As the dimensionality of the data increases, the computational cost of SVD rises exponentially, making these methods highly inefficient for handling large-scale matrices. To address this issue, researchers have proposed methods based on low-rank matrix decomposition such as go decomposition (GoDec) \cite{zhou2011godec}, semisoft GoDec (SSGoDec) \cite{zhou2013shifted}, low-rank matrix fitting (LMaFit) \cite{shen2014augmented}, and robust matrix factorization by majorization minimization (RMF-MM) \cite{lin2017robust}, etc.   
These methods commonly use the $\ell_1$ norm as a regularization term for enforcing sparsity. While the $\ell_1$ norm is favoured for its convexity, which makes it easier to optimize and performs well in constructing sparse solutions, it tends to excessively penalize large noise or outliers, leading to biased solutions. 
This bias problem becomes particularly pronounced when dealing with real-world data that contains significant outliers or noise. 
To address this issue, some non-convex penalty functions are introduced to the RPCA based on factorization.  
Recently, Wen {\it {et al.}}  employed generalized non-convex penalties for low-rank and sparsity in RPCA demonstrating improved performance under strict low-rank and sparsity conditions \cite{wen2019robust, wen2019nonconvex}.
Quach {\it {et al.}}  \cite{quach2017non} proposed an online robust PCA method using a non-convex $p$-norm,  which enhances performance without significant computational overhead and is effective for real-time tasks like background subtraction.  
Overall, with the introduction and development of non-convex regularization techniques, RPCA has seen significant improvements in handling large-scale data and noisy outliers, effectively mitigating the limitations of traditional $\ell_1$ regularization, and achieving better results in fields such as image processing, video analysis, and anomaly detection. 

Despite the advantages of non-convex regularization methods in reducing over-penalization and bias, they also have some drawbacks. First, non-convex regularization problems are generally challenging to optimize, as they involve non-convex functions, making it easy to get trapped in local optima and increasing the complexity of algorithm design and solution. Additionally, these methods can sometimes be overly sensitive to initial values or yield unstable solutions, particularly when dealing with highly noisy data, potentially preventing the attainment of stable and globally optimal solutions.
Motivated by the aforementioned challenges, we focus on a key issue: instead of relying on the $\ell_1$ norm or its non-convex approximation, why not adopt a simpler quadratic penalty while assigning different weights to individual data elements, thereby mitigating the bias introduced by the $\ell_1$ norm? 
To this end, we introduce a more streamlined model that not only avoids the bias induced by the $\ell_1$ norm but also addresses the complexities and instability associated with non-convex optimization. 

The main contributions of this work are summarized as follows:
\begin{enumerate}
\item We propose an RPCA model based on adaptive weighted least squares (AWLS) and low-rank matrix factorization (LRMF),  which employs a self-attention-inspired weighted Frobenius norm to represent the sparse components. This adaptive weighting strategy effectively reduces bias, simplifies the computation, and enhances robustness by dynamically adjusting penalties for different data elements.
\item The model is solved using an alternating minimization approach, where each subproblem has an explicit solution, significantly improving the computational efficiency. Meanwhile, we analyze the convergence of the proposed algorithm. 
\item Numerical experiments demonstrate that, despite its simplicity, the proposed model outperforms existing non-convex regularization methods. The self-attention-inspired weighting mechanism enhances both accuracy and robustness in the experimental results.
\end{enumerate}

The structure of this paper is organized as follows: In Section \ref{Pre}, we begin by introducing the notation and reviewing related work. In Section \ref{Proposed}, we present the proposed new RPCA model and outline the corresponding numerical algorithm.
In Section\ref{Converg}, the convergence of the proposed algorithm is given.
Section \ref{Numerical} presents the results of the numerical experiments. Finally, a brief conclusion is provided in Section \ref{Con}.

\section{Preliminaries}\label{Pre}
\subsection{Notations}
In this paper, we use the following notations:

\begin{itemize}
\item $\bX \in \mathbb{R}^{m \times n}$ represents a matrix with $m$ rows and $n$ column. $\bX_{ij} $ denotes the $(i,j)$ entry of the matrix $\bX$. 
$rank(\bX)$ denotes the rank of  matrix $\bX$. 
\item  The SVD of matrix $\bX$ is:
$
\bX = \bU \Sigma \bV^T,
$
where $\bU \in \R^{m \times m}$ and $\bV \in \R^{m\times n}$ are orthogonal left singular matrices and right singular matrices respectively; $\Sigma \in \R^{m \times n} $ is the diagonal matrix with $\Sigma_{ii} = \sigma_i(\bX)$, where  $\sigma_i(\bX)$  is the singular value of $\bX$. 
 $\|\bX\|_*$ is the nuclear norm of $\bX$, which is defined as:
\[
\|\bX\|_* = \sum_{i=1}^{\min\{m,n\}} \sigma_i(\bX).
\]  

\item $\|\bX\|_F$ denotes the Frobenius norm of a matrix $\bX$ defined as: 
\[
\|\bX\|_F = \sqrt{\sum_{i=1}^m \sum_{j=1}^n \bX_{ij}^2}.
\]
 \item $\|\bX\|_1$ denotes the $\ell_1$ norm of a matrix $\bX$ defined as: 
\[
\|\bX\|_1 = \sum_{i=1}^m \sum_{j=1}^n |X_{ij}|,
\]
which is the sum of the absolute values of all elements in $\bX$.
\item $\|\bX\|_0$ denotes the $\ell_0$ norm of a matrix $\bX$, defined as:
\[
\|\bX\|_0 = \text{card} \{ (i, j) \mid X_{ij} \neq 0 \},
\]
where ``$\text{card}$'' denotes the cardinality of the set, which counts the number of non-zero entries in $\bX$.
\end{itemize}
\subsection{Related work}
In this subsection, we review several classical models of RPCA that are relevant in this paper. For more information about RPCA, see \cite{bouwmans2016handbook}.  The original RPCA model focuses on minimizing the rank of the matrix $\bX \in \R^{m\times n}$ and the $\ell_0$ norm of the sparse component $\bS\in \R^{m\times n}$ \cite{candes2011robust}:
\begin{equation}\label{Rank}
	\begin{aligned}
		\min_{\bX, \bS} \quad & \rank(\bX) + \lambda \|\bS\|_0 \\
		\text{s.t.} \quad & \bY = \bX + \bS. 
	\end{aligned}
\end{equation}
However, directly minimizing the rank function and the $\ell_0$ norm is computationally challenging. Typically, the nuclear norm and the $\ell_1$ norm are employed as convex relaxations for the rank function and the $\ell_0$ norm, respectively \cite{wright2009robust,wright2010dense}. This leads to the following convex relaxation of \eqref{Rank}: 
\begin{equation} \label{Nuclear}
	\begin{aligned}
		\min_{\bX, \bS} \quad & \|\bX\|_* + \lambda \|\bS\|_1 \\
		\text{s.t.} \quad & \bY = \bX + \bS. 
	\end{aligned}
\end{equation}
The primary limitation of model \eqref{Nuclear} stems from the requirement to compute SVD at each iteration. As data dimensionality grows, the computational burden of SVD increases dramatically, rendering these approaches inefficient for processing large-scale matrices. This challenge has motivated the development of decomposition-based methods to overcome the scalability issues. In other approaches, RPCA is framed as a matrix factorization problem \cite{zhou2013greedy}, where the goal is to approximate $\bY$ by the product of two low-rank matrices $\bU$ and $\bV$ plus a sparse matrix $\bS$, {\it i.e.}, considering the following problem:  
\begin{equation}
	\min_{\bU, \bV, \bS} \quad \|\bY - \bU \bV - \bS\|_F^2 + \lambda \|\bS\|_1.
\end{equation}
To avoid biased estimates introduced by the $\ell_1$ norm, this model has been extended by replacing the $\ell_1$ norm with an alternative norm $\|\cdot\|_\phi$ for the sparse component \cite{wang2023robustHQF}: 
\begin{equation}\label{HQF}
	\min_{\bU, \bV, \bS} \quad \|\bY - \bU \bV - \bS\|_F^2 + \lambda \|\bS\|_\phi, 
\end{equation}
where $\|\bS\|_\phi = \sum_{i=1}^m \sum_{j=1}^n  \phi(\bS_{ij})$, and $\phi$ is a non-convex sparsity-promoting function  such as  $\ell_q$ norm \cite{quach2017non,marjanovic2012l_q}, MCP \cite{Zhang2010NearlyUV}, SCAD \cite{fan2001variable}, and so on. 

\section{Proposed Model and Algorithm} \label{Proposed}
%\subsection{Weight $\ell_2$ norm for sparse component $S$}
{In sparse matrix $\bS$, the number of non-zero elements is significantly smaller than the total number of elements, and their distribution is irregular. When the numerical values or distribution characteristics of these non-zero elements deviate notably from the norm of other non-zero elements in the matrix, they can be considered outliers. Outlier noise generally stems from observations within a dataset that deviate from the norm, potentially due to factors such as data entry errors, measurement inaccuracies, peculiar data distribution characteristics (like extreme events), or random fluctuations. }

A common statistical approach to handling outliers is the weighted least squares (WLS) method \cite{gao2016penalized}. Specifically, this model minimizes the weighted sum of squares of residuals:
\[
\min_{\bS = \bY - \bX} \sum_{ij} \bW_{ij} \bS_{ij}^2,
\]
where $W_{ij}$ represents the weights. {These weights $W_{ij}$ are either less than one (indicating suspected outliers) or equal to one (indicating non-outliers). Suspected outliers are assigned smaller weights to mitigate their influence. The weight matrix $\bW$ assigns differential weights to various entries of the sparse component $\bS$, enabling the model to down-weight suspected outliers and enhance robustness.
}

{
In practical applications, the data matrix $\bX$ is often not purely low-rank but can be approximated as the sum of a low-rank matrix $\bU\bV$ and a small F-norm matrix $\bE$, {\it i.e.}, $\bX = \bU\bV + \bE$, with $\|\bE\|_F^2$ typically small \cite{roy2024robust}. Therefore, the proposed Robust Principal Component Analysis (RPCA) model with WLS is expressed as:
\begin{equation} \label{RPCA-WLS}
\min_{\bU, \bV, \bS} \quad \|\mathbf{Y} - \bU\bV - \bS\|_F^2 + \lambda \|\bW \circ \bS\|_F^2,
\end{equation}
where $\lambda$ is a regularization parameter, and $ \circ $ denotes the Harmard multiplication. 
}

In an ideal scenario, if we know in advance which positions in the sparse component  $\bS$ are zero, we could simply assign a weight of 1 to these positions and 0 to the others. For instance, when an image is corrupted by impulse noise, we can model it as an RPCA problem, where the weights of the pixels damaged by noise are set to 0, and the weights of the unaffected pixels are set to 1. This approach transforms the denoising problem into a low-rank matrix completion problem.
However, in most cases, we do not have prior knowledge of the sparse component. Therefore, we need to design a mechanism to estimate appropriate weights that can identify which positions in the sparse component are zero. By setting different weights for different pixels, this approach can effectively avoid over-penalizing anomalous points and enhance the robustness of the model.
We first introduce how to solve the quadratic problem \eqref{RPCA-WLS}.

\subsection{Alternating Minimization Method}

%\tr{$\bW^k$改成 $\bW$. 在更新$\bW$后再重新写迭代过程。}

In this section, we adopt an alternating minimization approach \cite{wang2008new} to solve for $\bU,\bV,\bS$. 
This method iteratively updates each variable while keeping the others fixed, allowing us to decompose the original complex optimization problem into simpler sub-problems. Given $\bU^k,\bV^k$,  and $\bW$, 
we first consider the $\bS$ sub-problem: 
\begin{equation*}
	\bS^{k+1}=\argmin_{\bS} \|\bY-\bU^{k}\bV^{k}-\bS\|_F^2+\lambda \|\bW \circ \bS\|_F^2.
\end{equation*}  
Since this is a quadratic optimization problem for $\bS$, by taking the derivative of $\bS_{ij}$ and setting it to zero, we can easily derive the update formula for $\bS_{ij}$: 
\begin{equation} \label{Sk}
		\bS^{k+1}_{ij}= (\bY_{ij}-(\bU^{k}\bV^{k})_{ij})./(1+\lambda \bW_{ij}^{2}). 
\end{equation}
Then, given $\bS^{k+1}$, we take the proximal block coordinate descent method \cite{wen2019nonconvex} to update $\bU$ and $\bV$, {\it i.e.}, solving the following problem: 
\begin{equation}\label{UU}
	\begin{aligned}
		\argmin _{\bU} \left\|\bY-\bU \bV^k-\bS^{k+1}\right\|_F^2 +t\left\|\bU-\bU^k\right\|_F^2, 
	\end{aligned}
\end{equation}
\begin{equation}\label{VV}
	\begin{aligned}
		\argmin _{\bV}  \left\|\bX-\bU^{k+1} \bV-\bS^{k+1}\right\|_F^2+ t\left\|\bV-\bV^k\right\|_F^2, 
	\end{aligned}
\end{equation}
where $t>0$ is the proximal parameter. 
Given $\bS^{k+1}, \bV^{k}$, the optimality condition of \eqref{UU} is: 
$$(\bY-\bU\bV^{k}-\bS^{k+1})(\bV^{k})^{T}+t(\bU-\bU^k)=0.$$
Then we have
\begin{small}
\begin{equation}\label{Uk}
\bU^{k+1}= [t\bU^{k}+(\bY-\bS^{k+1})(\bV^{k})^T][\bV^k(\bV^k)^T+tI]^{-1}. 
\end{equation}
\end{small}
 Similarly, the optimality condition  of \eqref{VV} is: 
$$
-(\bU^{k+1})^{T}((\bY-\bU^{k+1}\bV-\bS^{k+1}))+t(\bV-\bV^{k})=0. 
$$
Then we get 
\begin{small}
\begin{equation}\label{Vk}
\begin{aligned}
\bV^{k+1} = [t\bI +(\bU^{k+1})^T\bU^{k+1}]^{-1}[t\bV^{k} +({\bU^{k+1})^T(\bY-\bS^{k+1})}]. 
\end{aligned}
\end{equation}
\end{small}

\subsection{{Self-attention mechanism for weight \texorpdfstring{$\bW$}{W} }}

{
In this model, the weight $\bW$ are crucial for identifying and down-weighting outliers, thereby enhancing the robustness of the RPCA framework. Ideally, if the positions of zeros in $\bS$ are known beforehand, one could assign weights of 1 to these positions and 0 to others. For instance, in image denoising, pixels corrupted by impulse noise can be identified, and their weights set to 0, transforming the denoising problem into a low-rank matrix completion problem. However, such prior knowledge is often unavailable in practice, necessitating a mechanism to estimate appropriate weights.
}
% 步骤1：计算动态缩放因子 $t^k$

% 首先，我们基于当前权重和稀疏组件计算一个动态缩放因子 $t^k$。这个操作类似于自注意力机制中的相似性计算，其中每个 $\bS^k$ 的元素都根据其重要性进行缩放，从而使模型能够聚焦于更有影响力的元素。公式如下：

% $$t^{k} = \frac{| \bW^{k} \circ \bS^{k} |}{\| \bW^{k} \circ \bS^{k} \|_\infty}$$

% 其中 $| \circ |$ 表示绝对值操作，$\| \circ \|_\infty$ 表示所有元素中的最大值。

% 步骤2：更新中间权重 $\widehat{\bW}^{k+1}$

% 在计算出动态缩放因子 $t^k$ 后，我们根据缩放后的 $\bS^k$ 值更新中间权重。这一步调整了权重，有效地强调了受到较少惩罚的元素，从而避免了过度惩罚异常值。这与自注意力机制中某些元素获得更高注意力分数的做法相似。公式如下：

% $$\widehat{\bW}^{k+1} = 1-(\vt^k)^p$$

% 其中 $p > 0$ 是一个预定义的参数。

% 步骤3：更新权重矩阵 $\bW^{k+1}$

% 最后，我们通过将之前计算的权重与现有权重相结合来获得更新后的权重矩阵 $\bW^{k+1}$。这一步类似于自注意力机制中的聚合阶段，其中权重通过迭代细化以增强模型的鲁棒性。通过为不同元素分配自适应惩罚，这种方法有效地防止了异常值的过度惩罚，从而提高了模型的稳定性和鲁棒性。公式如下：

% $$\bW^{k+1} = \widehat{\bW}^{k+1} \circ \bW^{k}$$

% 综上所述，该权重更新机制确实可以视为一种自注意力机制的形式，因为它通过动态调整权重矩阵来关注更重要的信息，并捕获元素之间的关系，从而增强了模型的鲁棒性和性能。

Next, we describe the procedure to update the weight matrix $ \bW $. %In practical applications, we typically lack prior knowledge of the sparse component $ \bS $. Therefore, it becomes essential to devise a mechanism to estimate suitable weights that can effectively distinguish which positions in $ \bS $ should be zero. 
%To address this, 
We employ the strategy introduced in \cite{gazzola2020inner}, where a weighted least squares model was proposed for image restoration and reconstruction. In that model, different weights are assigned to other pixels, allowing for adaptive penalization. This effectively avoids excessive penalization of outliers, thereby improving the model’s robustness. Moreover, the model preserves critical edge structures by applying adaptive penalties to image gradients. Inspired by this approach, we apply an adaptive penalty to the sparse components within our RPCA model, enhancing its capacity to handle outliers while maintaining key structural information.

{{This weight update mechanism can be interpreted as a form of self-attention, as it dynamically adjusts the weight matrix $ \bW $ in each iteration based on the current state of the sparse component $ \bS^k $. }
Self-attention is a mechanism widely used in deep learning, especially in models like Transformers, to focus on different parts of the input data and capture relationships between elements \cite{bahdanau2014neural,vaswani2017attention}. 
The following steps illustrate how this self-attention-like mechanism operates through the update formulas. 
Given  $\bS^{k+1}$ and $\bW$, the process to determine the weight matrix $ \bW^{k+1} $ is as follows:\\
\textbf{Step 1: Calculation of the dynamic scaling factor $ t^k $.}
To begin, we compute a dynamic scaling factor $ t^k $ based on the current weights and sparse component, as shown below:
\begin{equation} \label{W1}
    t^{k} = \frac{| \bW^{k} \circ \bS^{k} |}{\| \bW^{k} \circ \bS^{k} \|_\infty},
\end{equation}
where $ | \cdot | $ denotes the absolute value operation, and $ \| \cdot \|_\infty $ represents the maximum value across all elements. This operation is analogous to calculating similarity in self-attention mechanisms, where each element of $ \bS^k $ is scaled relative to its importance, focusing the model on more influential elements.
 \\
\textbf{Step 2: Update of intermediate weight $ \widehat{\bW}^{k+1} $.}
With the dynamic scaling factor $ t^k $ computed, we proceed to update the intermediate weights as follows:
\begin{equation}\label{W2}
\widehat{\bW}^{k+1} = 1-(\vt^k)^p, 
\end{equation}
where $ p > 0 $ is a predefined parameter. This updated step adjusts the weights based on the scaled values of $ \bS^k $, effectively emphasising less penalised elements, thereby avoiding excessive punishment of outliers, akin to self-attention where certain elements receive higher attention scores. \\
\textbf{Step 3: Update of the weight matrix $ \bW^{k+1} $.}
Finally, the updated weight matrix $ \bW^{k+1} $ is obtained by combining the previously computed weights with the existing weight, {\it i.e.}, 
\begin{equation}\label{W3}
	\bW^{k+1} =\widehat{\bW}^{k+1} \circ \bW^{k}.
\end{equation}
This final step resembles the aggregation phase in self-attention, where weights are refined iteratively to enhance the model’s robustness. By assigning adaptive penalties to different elements, this approach effectively prevents the over-penalization of outliers, thus increasing the stability and robustness of the model.
}
The process of solving problem \eqref{RPCA-WLS} is summarized as Algorithm 1.

\begin{algorithm} 
\SetAlgoNoLine
\caption{RPCA based on AWLS and LRMF.}
\label{Alg1}
\KwIn{ $\bY$}
%\KwOut{}
Initialize $\bW^{0},\bU^0, \bV^0, t$\; 
\For {$k=0,1,2,\ldots$, until the stopping criterion is satisfied} 
{
Update $\bW^{k}$ by \eqref{W1}--\eqref{W3}\;
        Calculate $\bS^{k+1}$ by \eqref{Sk}\;
        Calculate $\bU^{k+1}$ by \eqref{Uk}\;
        Calculate $\bV^{k=1}$ by \eqref{Vk}\;
}
\Return $\bU= \bU^{k+1},\bV= \bV^{k+1},\bS= \bS^{k+1}$. 
\end{algorithm}

\section{Convergence Analysis} \label{Converg}

Next, we will analyze the convergence behaviour of the weight matrix. Given the normalization process and the application of the absolute value operation, the entries of the weight vector are constrained within the interval $[0, 1]$, that is, $0 \leq 	\widehat{\bW}^{k}_{ij} \leq 1$. Considering the initial condition $0 \leq [\bW^{0}]_{ij} \leq 1$, and based on the update formula $\bW^{k} = 	\widehat{\bW}^k \circ \bW^{k-1}$, it can be readily verified that as the number of iteration steps increases, the diagonal entries of the weight matrix defined in \eqref{W3} are non-increasing. Specifically, the $i$-th diagonal entry of two consecutive weight matrices satisfies 
\begin{equation}\label{Wii}
[\bW^{k}]_{ii} \leq [\bW^{k-1}]_{ii}.
\end{equation}
Consequently, based on the definition of the weight matrix, we arrive at the following lemma.
\begin{lemma}\label{ConW}
If $\bW^{k}$ is generated by \eqref{W1}-\eqref{W3}, then there exists a diagonal matrix $\bW^{*}$ such that 
$$
\lim\limits_{k \to \infty} \bW^{k} = \bW^{*}. 
$$
\end{lemma}

Lemma \ref{ConW} asserts that the sequence $\bW^{k}$ converges as the iteration progresses. This result is crucial for establishing the convergence of the proposed algorithm.

Let us first denote 
\begin{equation}\label{cJ}
\cJ(\bU,\bV, \bS;\bW) := \frac{1}{2} \|\bY-\bU \bV-\bS\|_F^2+\lambda\|\bW \circ \bS\|_F^2,  
\end{equation}
{\it c.f.} (\ref{RPCA-WLS}). 
We now analyze the convergence of Algorithm \ref{Alg1}. First, we have the following theorem.
\begin{theorem}\label{ThJ}
The sequence $\{\bU^{k}, \bV^{k}, \bS^{k};\bW^{k}\}$ generated by Algorithm \ref{Alg1},  satisfies the following properties:
\begin{enumerate}
\item The objective function of $\cJ$ is non-increasing.  Specifically, we have 
 \begin{equation}
\begin{aligned}
&\cJ(\bU^{k}, \bV^{k}, \bS^{k};\bW^{k}) - \cJ(\bU^{k+1}, \bV^{k+1}, \bS^{k+1};\bW^{k+1}) \\
&\geq  t\left\|\boldsymbol{U^{k+1}}-\bU^{k}\right\|_F^2 + t\left\|\bV^{k+1}-\bV^{k}\right\|_F^2. 
\end{aligned}
\end{equation}
\item  The sequence $\bU^k, \bV^k$ is convergence, {\it i.e.}, 
$$
\begin{aligned}
\lim\limits_{k\to{\infty}} \|\bU^{k+1}  - \bU^k\|_F =0, \\
 \lim\limits_{k\to{\infty}} \|\bV^{k+1}  - \bV^k\|_F =0. 
\end{aligned}
$$
\end{enumerate}

\end{theorem}
\begin{proof} 
We first show  that when updating $W^{k+1}$ by \eqref{W3},  
   \begin{equation}\label{cJ1}
 \cJ(\bU^{k}, \bV^{k}, \bS^{k};\bW^{k+1}) 
\leq  \cJ(\bU^{k}, \bV^{k}, \bS^{k};\bW^{k})
\end{equation}
holds. According to \eqref{Wii}, for $\bS^{k}$, we have 
$$
\norm{\bW^{k+1}\circ \bS^{k}}_F^2 \leq \norm{\bW^{k} \circ \bS^{k}}_F^2.
$$  
Further, by (\ref{cJ}) we obtain
\begin{equation}
\begin{aligned}
&\cJ(\bU^{k}, \bV^{k}, \bS^{k};\bW^{k+1})-\cJ(\bU^{k}, \bV^{k}, \bS^{k};\bW^{k}) \\ &= \norm{\bW^{k+1}\circ \bS^{k}}_F^2 - \norm{\bW^{k}\circ \bS^{k}}_F^2 \leq 0. 
\end{aligned}
\end{equation}
Second, since 
$$ 
\bS^{k+1} = \argmin_{\bS} \cJ(\bU^{k}, \bV^{k},\bS;\bW^{k+1} ), 
$$
we have
   \begin{equation} \label{cJ2}
    \cJ(\bU^{k}, \bV^{k}, \bS^{k+1};\bW^{k+1}) 
   \leq \cJ(\bU^{k}, \bV^{k}, \bS^{k};\bW^{k+1}). 
  \end{equation}
   Then, considering that
   $$
   \bU^{k+1}=\arg \min _{\bU} \left\|\bY-\bU \bV^k-\bU^{k}\right\|_F^2+t\left\|\bU-\bU^k\right\|_F^2, 
   $$
   it's easy to get 
   \begin{equation}
   \begin{aligned}
   & \left\|\bY-\bU^{k+1} \bV^{k}-\bU^{k+1}\right\|_F^2+t\left\|{\bU^{k+1}}-\bU^{k}\right\|_F^2 \\
   & \leq \left\|\bY-\bU^{k} \bV^{k}-\bU^{k+1}\right\|_F^2 
   \end{aligned}
   \end{equation}
   Add $\|\bW^{k+1} \circ  \bS^{k+1}\|_F^2$ to both sides of the above inequality, one can obtain
      \begin{equation} \label{cJ3}
   \begin{aligned}
   & \cJ(\bU^{k+1}, \bV^{k}, \bS^{k+1};\bW^{k+1}) \\ 
   \leq &\cJ(\bU^{k}, \bV^{k}, \bS^{k+1};\bW^{k+1})- t\left\|{\bU^{k+1}}-\bU^{k}\right\|_F^2. 
   \end{aligned}
   \end{equation}
   Similarly,  when updating $\bV^{k+1}$, we have 
         \begin{equation} \label{cJ4}
   \begin{aligned}
   & \cJ(\bU^{k+1}, \bV^{k+1}, \bS^{k+1};\bW^{k+1}) \\ 
   \leq &\cJ(\bU^{k+1}, \bV^{k}, \bS^{k+1};\bW^{k+1})- t\left\|\bV^{k+1}-\bV^{k}\right\|_F^2. 
   \end{aligned}
   \end{equation}
 By combining \eqref{cJ1}, \eqref{cJ2}, \eqref{cJ3} and \eqref{cJ4}, we establish the following conclusion
 {
 \begin{equation}\label{T1}
 \begin{aligned}
   &\cJ(\bU^{k}, \bV^{k}, \bS^{k};\bW^{k}) - \cJ(\bU^{k+1}, \bV^{k+1}, \bS^{k+1};\bW^{k+1}) \\
   &\geq  t\left\|{\bU^{k+1}}-\bU^{k}\right\|_F^2 + t\left\|\bV^{k+1}-\bV^{k}\right\|_F^2. 
 \end{aligned}
 \end{equation}} 
 %Then, from \ref{T1}, we can conclude that 
 We can therefore conclude that the sequence $\cJ(\bU^{k}, \bV^{k}, \bS^{k};\bW^{k})$ is non-increasing and convergent since $\cJ$ is lower bounded.
 Then, summing the inequality \eqref{T1} from $k=0$ to $p-1$, we easily get 
 $$
 \begin{aligned}
&\sum_{k=0}^{p-1}  \cJ(\bU^{k}, \bV^{k}, \bS^{k};\bW^{k}) - \cJ(\bU^{k+1}, \bV^{k+1}, \bS^{k+1};\bW^{k+1})\\
&=\cJ(\bU^{0}, \bV^{0}, \bS^{0};\bW^{0}) -  \cJ(\bU^{p}, \bV^{p}, \bS^{p};\bW^{p}) \\
& \geq \sum_{k=0}^{p-1} \left(t\left\|\bU^{k+1}-\bU^{k}\right\|_F^2+ t\left\|\bV^{k+1}-\bV^{k}\right\|_F^2 \right). 
 \end{aligned}
 $$
Let $p\to \infty$, we have 
$$
 \begin{aligned}
& \sum_{k=0}^{\infty} \left(\left\|\bU^{k+1}-\bU^{k}\right\|_F^2+ \left\|\bV^{k+1}-\bV^{k}\right\|_F^2 \right)\\
& \leq \frac{1}{t} \left( \cJ(\bU^{0}, \bV^{0}, \bS^{0};\bW^{0}) -  \cJ(\bU^{p}, \bV^{p}, \bS^{p};\bW^{p}) \right) \\
& < \infty. 
 \end{aligned}
$$
Therefore, we obtain
$$
\begin{aligned}
&\lim\limits_{k\to{\infty}} \|\bU^{k+1}  - \bU^k\|_F =0,  \\
&\lim\limits_{k\to{\infty}} \|\bV^{k+1}  - \bV^k\|_F =0,  
\end{aligned}
$$
which completes the proof. 
\end{proof}
Then, we can easily arrive at the following conclusion.
\begin{theorem}
The limit point of $(\bU^k,\bV^k)$ is the minimum of the minimization problem \eqref{RPCA-WLS}. 
\end{theorem}
\begin{proof}
It's obvious that 
$$
\begin{aligned}
\partial_{\bU} \cJ(\bU, \bV, \bS;\bW) =(\bY - \bU\bV - \bS)\bV^T, \\
\partial_{\bV} \cJ(\bU, \bV, \bS;\bW) =\bU^T(\bY - \bU\bV - \bS). 
\end{aligned}
$$ 
Now, we prove that
\begin{equation} \label{Opt}
\begin{aligned}
	\partial_{\bU} \cJ(\bU^*, \bV^*, \bS^k;\bW^k) =0, \\
	\partial_{\bV} \cJ(\bU^*, \bV^*, \bS^k;\bW^k)=0. 
\end{aligned}
\end{equation}
is true.
Considering the $\bU$, $\bV$ subproblems \eqref{UU} and \eqref{VV}, we have
$$
\partial_{\bU} \cJ(\bU^{k}, \bV^{k-1}, \bS^{k};\bW^{k})+ t (\bU^{k} -\bU^{k-1}) =0,
$$
$$
\partial_{\bV} \cJ(\bU^{k}, \bV^{k}, \bS^{k};\bW^{k})+t (\bV^{k} -\bV^{k-1})=0. 
$$
Further, we have
\begin{equation}\label{ParU}
	\begin{aligned}
	& \partial_{\bU}   \cJ(\bU^{k}, \bV^{k}, \bS^{k};\bW^{k}) \\
	& = \partial_{\bU} \cJ(\bU^{k}, \bV^{k}, \bS^{k};\bW^{k}) - \partial_{\bU} \cJ(\bU^{k}, \bV^{k-1}, \\ &\qquad  \bS^{k};\bW^{k})
	-  t (\bU^{k} -\bU^{k-1}) \\
	& = X(\bV^{k}-\bV^{k-1})^T -\bU^k(\bV^{k}-\bV^{k-1})(\bV^{k}\\
	& \quad +\bV^{k-1})^T-\bS^k(\bV^{k}-\bV^{k-1})^T 	-  t (\bU^{k} -\bU^{k-1}) , 
	\end{aligned}
\end{equation}
\begin{equation}\label{ParV}
	\partial_{\bV} \cJ(\bU^{k}, \bV^{k}, \bS^{k};\bW^{k})=- t (\bV^{k} -\bV^{k-1}).
\end{equation}
Suppose $\{\bU^{k_j}\}$ and $\{\bV^{k_j}\}$ is the bounded subsequence of $\{\bU^{k}\}$ and $\{\bV^{k}\}$ , respectively, which means  
$$
\lim\limits_{k_j\to \infty} \bU^{k_j} = \bU^*,  \lim\limits_{k_j\to \infty} \bV^{k_j} = \bV^* . 
$$
Obviously, $\bU^{k_j}$  and $\bV^{k_j}$ are  also satisfy \eqref{ParU} and \eqref{ParV}. When $k_j\to \infty$, then \eqref{Opt} holds. 
Therefore, we can conclude that $(\bU^*, \bV^*)$ is the minimum of the  problem \eqref{RPCA-WLS}. 

\end{proof}

\begin{figure}
	\centering \hspace{-0.65cm} 
	\subfloat[]{ 
		\includegraphics[width=0.38\linewidth]{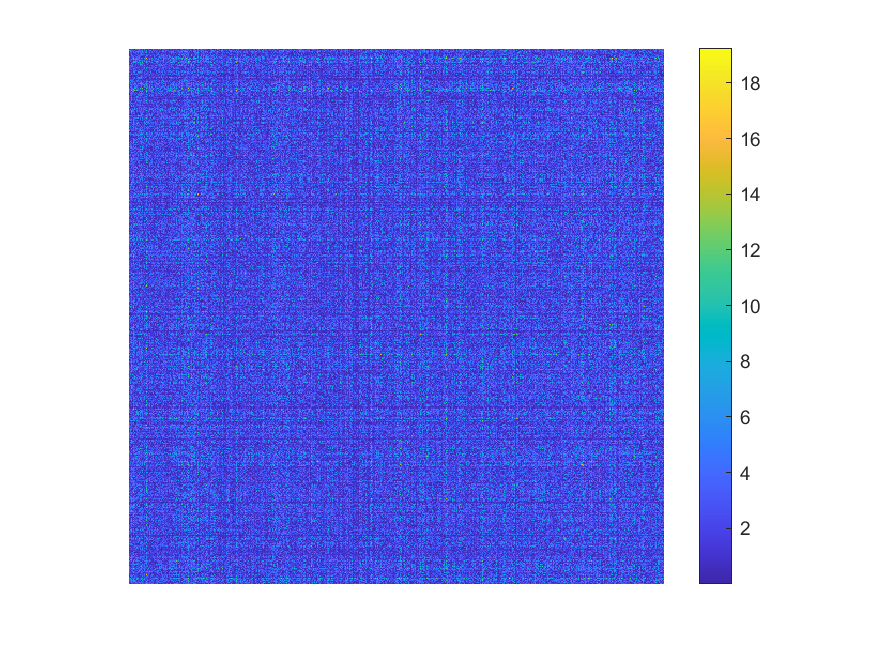} \label{F_1} \hspace{-0.78cm}  } 
	\subfloat[]{ 
		\includegraphics[width=0.38\linewidth]{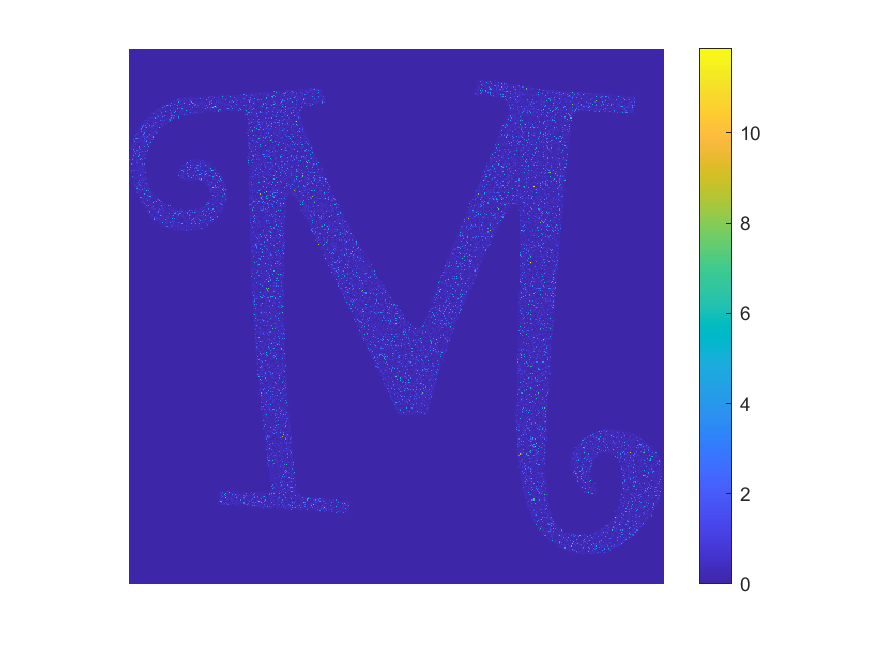} \label{F_2} \hspace{-0.78cm} }
	\subfloat[]{ 
		\includegraphics[width=0.38\linewidth]{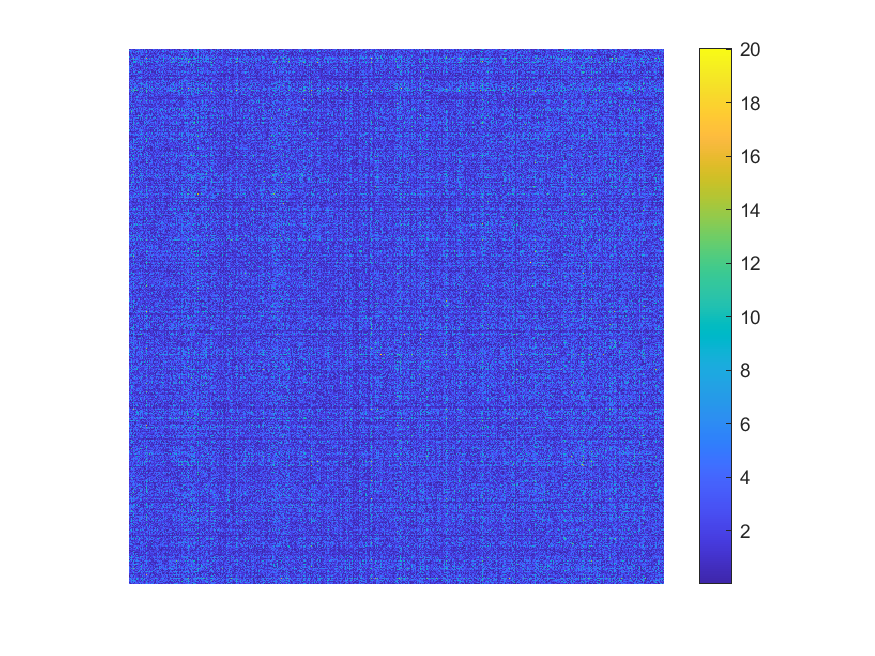} \label{F_3} }
	%	\vspace{-0.3cm}
	
	\hspace{-0.65cm} 
	\subfloat[]{ 
		\includegraphics[width=0.38\linewidth]{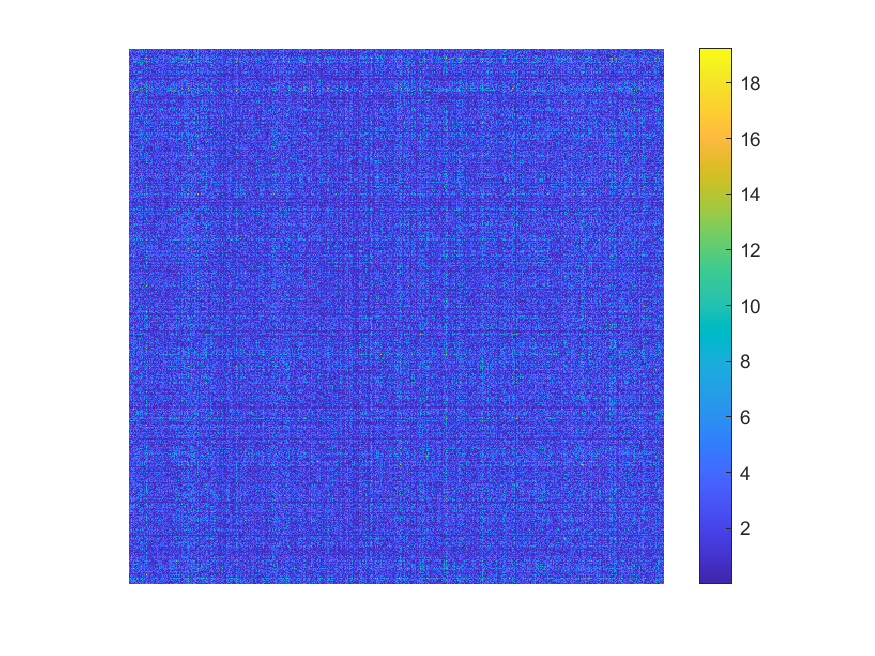} \label{F_4} \hspace{-0.78cm} } 
	\subfloat[]{ 
		\includegraphics[width=0.38\linewidth]{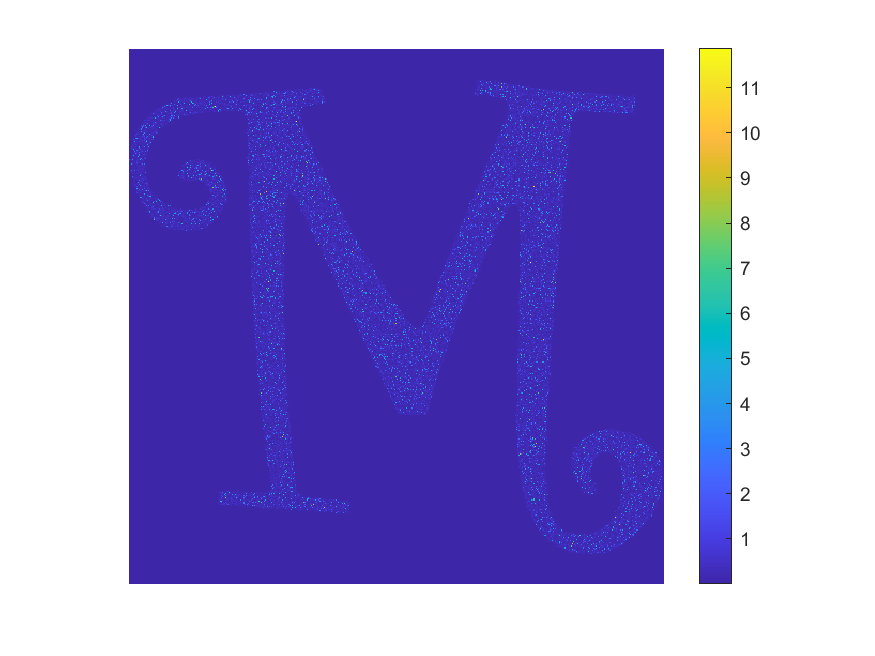} \label{F_5} \hspace{-0.78cm} }
	\subfloat[]{ 
		\includegraphics[width=0.38\linewidth]{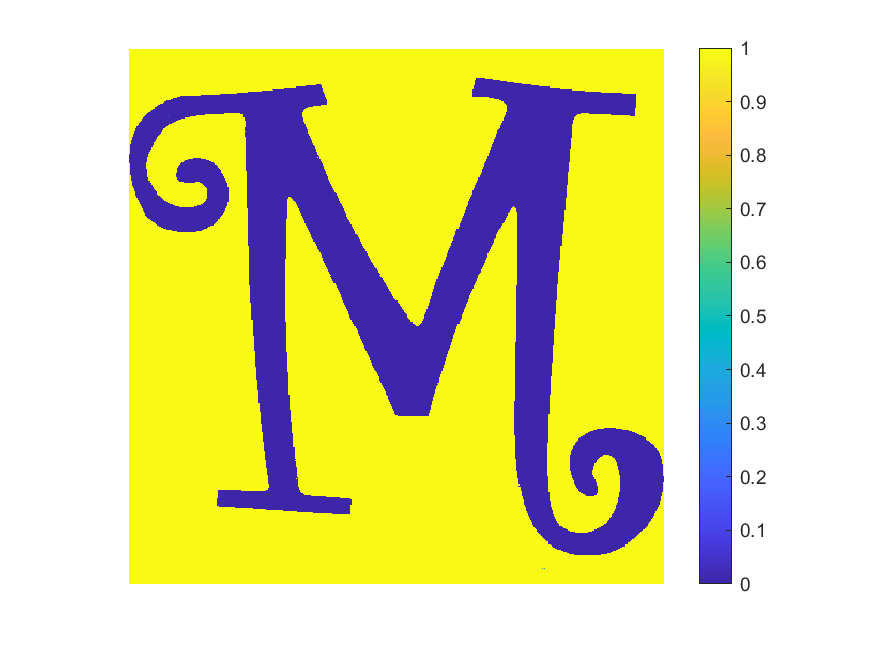} \label{F_6}	}
	\caption{Results of the proposed method tested on synthetic data. Rank is 10,  sparsity is  $26.4\%$, the RMSE of the low-rank matrix is 7.2e-09, and the RMSE  of the sparse matrix is 1.1e-08. (a) Low-rank component $\bX_{\text{true}}$;  (b) Sparse component $\bS_{\text{true}}$; (c) $\bX_{\text{true}}+ \bS_{\text{true}}$; (d) Separated $\hat{\bX}$; (e) Separated $\hat{\bS}$; (f) Estimated $\bW_k$.   }
	\label{WW}
\end{figure}

\begin{figure}
	\centering  \hspace{-0.65cm} 
	\includegraphics[width=0.38\linewidth]{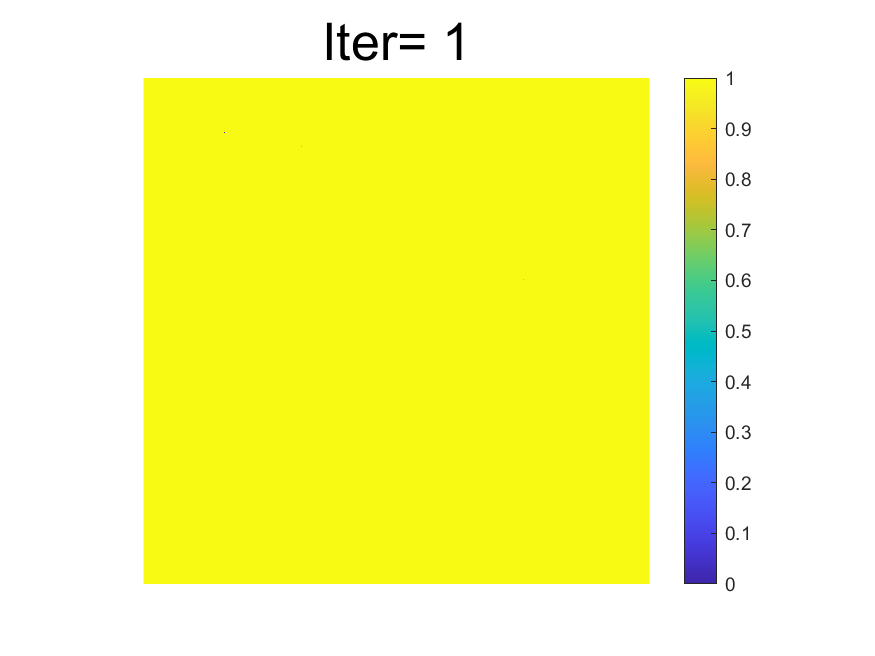}  \hspace{-0.68cm} 
	\includegraphics[width=0.38\linewidth]{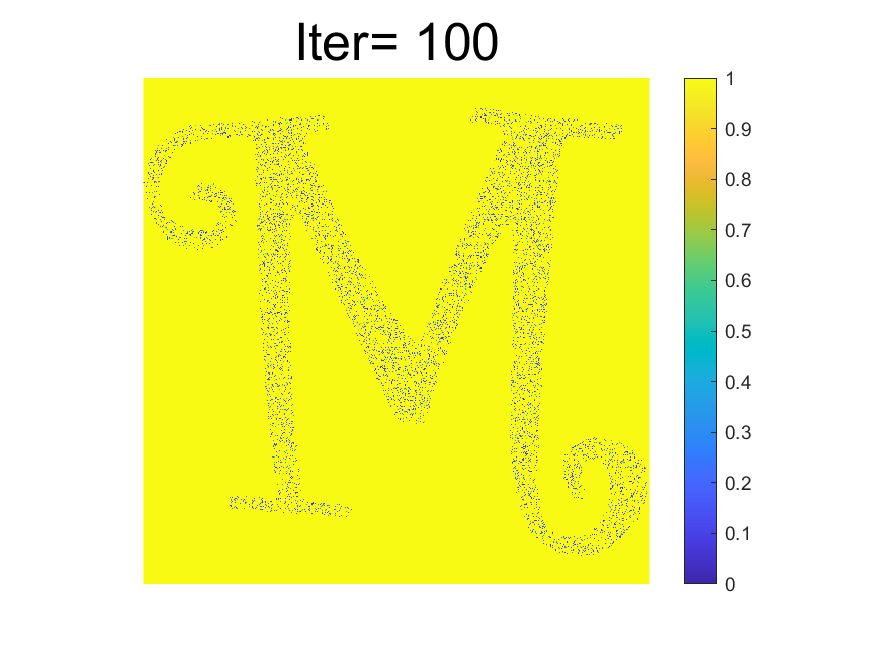} \hspace{-0.68cm} 
        \includegraphics[width=0.38\linewidth]{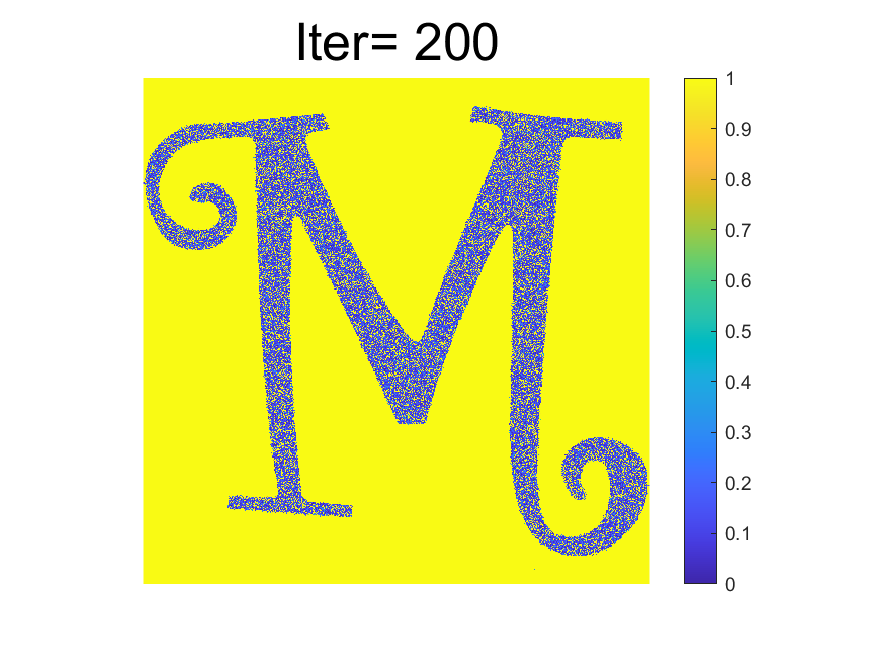} \hspace{-0.68cm} 

        \hspace{-0.65cm} 
        \includegraphics[width=0.38\linewidth]{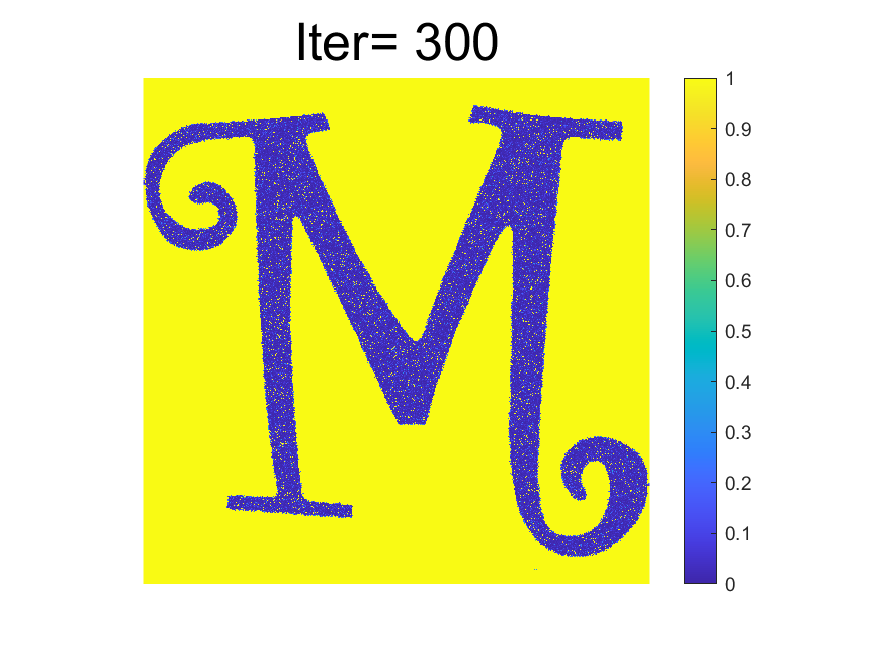}  \hspace{-0.68cm} 
	\includegraphics[width=0.38\linewidth]{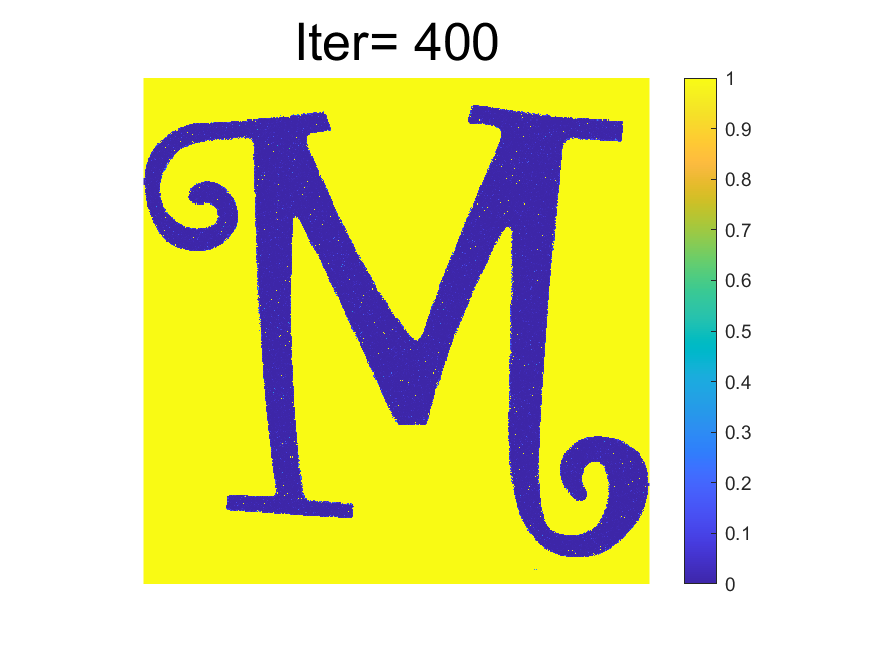} \hspace{-0.68cm} 
        \includegraphics[width=0.38\linewidth]{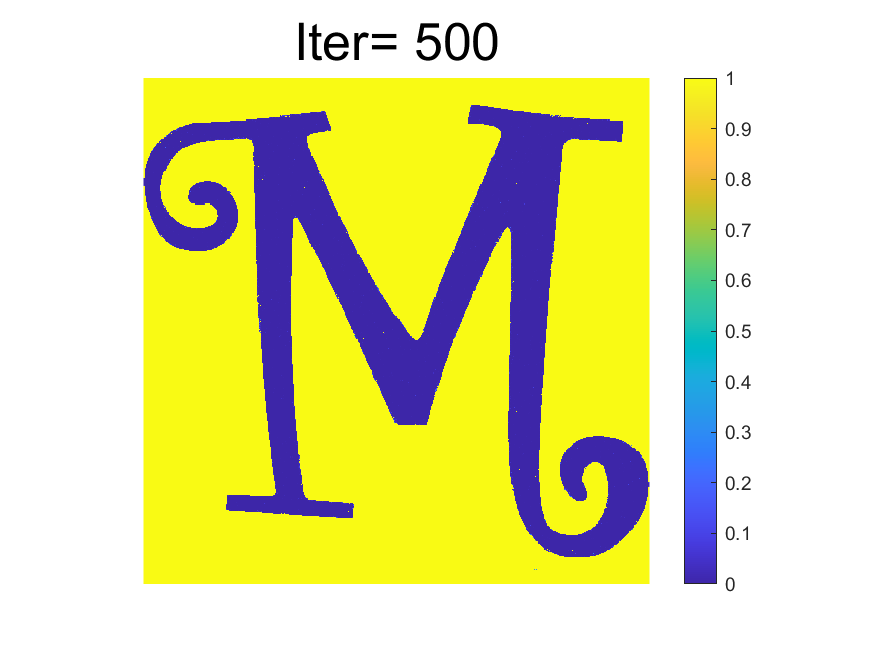}  \hspace{-1cm}   
    \caption{Evolution of the Weight Matrix During Iterations in the proposed model.}
    \label{Wkk}
\end{figure}

\renewcommand\arraystretch{1.3}
\begin{table*}[!ht]
    \centering
    \caption{The root mean square error (RMSE) of the low-rank matrix  $\bX$ separated using different methods at a sparsity level of $10\%$ under varying SNRs. }
    \scalebox{0.75}{
        \begin{tabular}{|M{2cm}|M{1.6cm}|M{1.6cm}|M{1.6cm}|M{1.6cm}|M{1.6cm}|M{1.6cm}|M{1.6cm}|M{1.6cm}|}
            \hline
            \multicolumn{9}{|c|}{$m=n = 500$} \\
            \hline
            \multicolumn{1}{|c|}{SNR } & \textbf{GoDec+} & \textbf{WNNM} & \textbf{NC} & \textbf{OBC} & \textbf{HQF} & \textbf{HOW}& {\textbf{W-L0}} & {\textbf{W-L2}} \\ \hline
            \multicolumn{1}{|c|}{1} & 8.44e-04 & 5.91e-08 & 4.70e-05 & 1.36e-01 & 4.17e-04 & 2.77e-09 &4.94e-09  & 6.93e-11 \\
            \multicolumn{1}{|c|}{3} & 9.50e-04 & 7.19e-08 & 5.02e-05 & 1.13e-01 & 3.18e-04 & 2.51e-09 &  1.87e-09 & 8.27e-11 \\
            \multicolumn{1}{|c|}{6} & 1.15e-03 & 7.28e-08 & 1.01e-04 & 9.63e-02 & 2.32e-04 & 3.10e-08 &  5.31e-09  & 6.92e-11 \\
            \multicolumn{1}{|c|}{9} & 1.31e-03 & 5.51e-08 & 2.37e-04 & 8.19e-02 & 1.49e-04 & 2.42e-08 & 1.52e-09   & 6.49e-11 \\
            \multicolumn{1}{|c|}{12} & 1.54e-03 & 7.56e-08 & 9.06e-04 & 8.06e-02 & 1.23e-04 & 9.80e-09 & 1.89e-09   & 6.70e-11 \\
            \multicolumn{1}{|c|}{15} & 1.82e-03 & 4.21e-08 & 4.41e-03 & 7.74e-02 & 7.71e-05 & 1.54e-09 &  2.57e-09  & 8.01e-11 \\
            \hline
            Average Time (s) & 0.05  & 4.03  & 34.03  & 0.43  & 0.14  & 0.64     &  0.24 & 0.85  \bigstrut\\
            \hline
            \multicolumn{9}{|c|}{$m=n = 1000$} \\
            \hline
            \multicolumn{1}{|c|}{SNR } & \textbf{GoDec+} & \textbf{WNNM} & \textbf{NC} & \textbf{OBC} & \textbf{HQF} & \textbf{HOW} & {\textbf{W-L0}}  & {\textbf{W-L2}} \bigstrut\\
            \hline
            \multicolumn{1}{|c|}{1} & 5.95e-04 & 4.65e-08 & 1.34e-05 & 1.52e-01 & 2.94e-05 & 1.25e-09 &  1.75e-10 &7.12e-11 \\
            \multicolumn{1}{|c|}{3} & 6.50e-04 & 4.39e-08 & 2.08e-05 & 1.50e-01 & 3.29e-05 & 9.09e-10 &  2.18e-10  & 7.55e-11 \\
            \multicolumn{1}{|c|}{6} & 7.82e-04 & 4.11e-08 & 4.26e-05 & 1.41e-01 & 2.63e-05 & 5.17e-10 & 3.73e-10  &8.73e-11 \\
            \multicolumn{1}{|c|}{9} & 9.11e-04 & 4.43e-08 & 1.11e-04 & 1.42e-01 & 1.65e-05 & 4.54e-10 &  1.14e-10 &6.57e-11 \\
            \multicolumn{1}{|c|}{12} & 1.07e-03 & 4.56e-08 & 4.70e-04 & 1.36e-01 & 1.23e-05 & 1.21e-10 &   7.71e-10 &7.73e-11 \\
            \multicolumn{1}{|c|}{15} & 1.23e-03 & 4.42e-08 & 2.40e-03 & 1.36e-01 & 9.96e-06 & 2.88e-09 &  2.30e-10 &6.11e-11 \\
            \hline
            Average Time (s) & 0.10  & 17.47  & 159.53  & 1.16  & 0.61  & 2.36  &  0.78 & 3.44  \bigstrut\\
            \hline
        \end{tabular}}%
    \label{Table1}%
\end{table*}%

\begin{table*}[!ht]
    \centering
    \caption{The root mean square error (RMSE) of the low-rank matrix  $\bX$ separated using different methods at a sparsity level of $20\%$ under varying SNRs.}
    \scalebox{0.75}{
        \begin{tabular}{|M{2cm}|M{1.6cm}|M{1.6cm}|M{1.6cm}|M{1.6cm}|M{1.6cm}|M{1.6cm}|M{1.6cm}|M{1.6cm}|}
            \hline
            \multicolumn{9}{|c|}{$m=n = 500$} \\
            \hline
            \multicolumn{1}{|c|}{SNR } & \textbf{GoDec+} & \textbf{WNNM} & \textbf{NC} & \textbf{OBC} & \textbf{HQF} & \textbf{HOW} & {\textbf{W-L0}} &{\textbf{W-L2}} \\ \hline
            \multicolumn{1}{|c|}{1} & 4.64e-03 & 6.44e-08 & 2.10e-04 & 1.83e-01 & 4.53e-05 & 9.05e-11 & 9.53e-11 &8.59e-11 \\
            \multicolumn{1}{|c|}{3} & 8.07e-03 & 8.29e-08 & 4.96e-04 & 1.41e-01 & 4.19e-05 & 1.17e-10 & 5.10e-10 &6.97e-11 \\
            \multicolumn{1}{|c|}{6} & 1.83e-03 & 8.52e-08 & 1.99e-03 & 1.15e-01 & 6.68e-05 & 6.05e-09 & 3.50e-9 &7.06e-11 \\
            \multicolumn{1}{|c|}{9} & 2.15e-03 & 7.28e-08 & 1.22e-02 & 9.49e-02 & 1.31e-05 & 6.02e-11 & 5.36e-11 &7.31e-11 \\
            \multicolumn{1}{|c|}{12} & 2.52e-03 & 8.34e-08 & 2.79e-02 & 8.30e-02 & 1.93e-05 & 4.46e-08 & 1.31e-07 & 8.74e-11 \\
            \multicolumn{1}{|c|}{15} & 2.86e-03 & 9.16e-08 & 3.70e-02 & 8.07e-02 & 8.31e-06 & 7.95e-10 & 7.19e-10 &7.50e-11 \\
            \hline
            Average Time (s) & 0.07  & 4.54  & 40.56  & 0.42  & 0.27  & 0.87  &0.39 & 0.93 \\
            \hline
            \multicolumn{9}{|c|}{$m=n = 1000$} \\
            \hline
            \multicolumn{1}{|c|}{SNR } & \textbf{GoDec+} & \textbf{WNNM} & \textbf{NC} & \textbf{OBC} & \textbf{HQF} & \textbf{HOW} & {\textbf{W-L0}}& {\textbf{W-L2}} \\ \hline
            \multicolumn{1}{|c|}{1} & 9.48e-04 & 5.54e-08 & 7.89e-05 & 1.72e-01 & 1.42e-05 & 4.37e-11 & 5.51e-11 & 9.28e-11 \\
            \multicolumn{1}{|c|}{3} & 1.05e-03 & 6.32e-08 & 1.65e-04 & 1.61e-01 & 7.87e-06 & 3.32e-10 & 4.89e-11 &8.92e-11 \\
            \multicolumn{1}{|c|}{6} & 1.24e-03 & 5.44e-08 & 9.80e-04 & 1.50e-01 & 4.55e-06 & 1.80e-11 & 4.03e-11 & 9.48e-11 \\
            \multicolumn{1}{|c|}{9} & 1.47e-03 & 5.05e-08 & 8.08e-03 & 1.42e-01 & 1.07e-05 & 7.36e-11 & 1.01e-10 & 9.15e-11 \\
            \multicolumn{1}{|c|}{12} & 1.72e-03 & 5.32e-08 & 2.38e-02 & 1.41e-01 & 3.83e-06 & 8.52e-10 & 4.42e-09 & 6.95e-11 \\
            \multicolumn{1}{|c|}{15} & 1.97e-03 & 5.41e-08 & 3.50e-02 & 1.40e-01 & 2.72e-06 & 6.04e-10 & 3.47e-09 &7.49e-11 \\
            \hline
            Average Time (s) & 0.15  & 18.32  & 207.52  & 1.16  & 0.97  & 2.67  & 1.18 & 3.67  \bigstrut\\
            \hline
        \end{tabular}}%
    \label{Table2}%
\end{table*}%

\section{Numerical experiments} \label{Numerical}
\subsection{The weight  \texorpdfstring{$\bW$}{W}. }
In this subsection, we will conduct a simulation experiment to investigate the role of the weight matrix in our proposed RPCA model \eqref{RPCA-WLS}.
First, a low-rank matrix $\bX$ is generated using  $\bX = \bU\bV$, where $\bU\in\mathbb{R}^{ 500 \times 10}$ and $\bV \in \mathbb{R}^{10 \times 500}$ are randomly generated matrices that follow a normal distribution. Next, the Gaussian noise with a SNR \footnote{The SNR is defined as $ \text{SNR} = \log_{10} \left(\frac{\|\bX\|_F^2}{\|\bM\|_F^2}\right),$ where $\bM$ is the Gaussian white noise.} 
of 10 is generated, and this noise is multiplied by a sparse mask to obtain $\bS$ (see Figure \ref{F_2}), where the sparsity of $\bS$ is 26.4\%.  Finally, the sparse noise is added to the low-rank matrix, resulting in a low-rank matrix $\bY$ with sparse Gaussian noise (see Figure \ref{F_3}). 
Subsequently, we consider the low-rank sparse decomposition of the observed matrix $ \bY $ to recover the original low-rank matrix $ \bX $ and the sparse component $ \bS $ from the observations. We use the root mean square error (RMSE) to quantitatively evaluate the separation performance, which is defined as
\[
\text{RMSE}(\bX) = \sqrt{\frac{1}{mn} \sum_{i=1,j=1}^{m,n} \left( \bX_{\text{true},i,j} - {\bX}_{i,j} \right)^2}. 
\]
The experimental results illustrated in Figure \ref{WW} provide a detailed insight into the effectiveness of the proposed low-rank and sparse decomposition method applied to the observation matrix $\bY$.  
The recovery of the low-rank component $\bX_{\text{true}}$, as shown in Figure \ref{F_4}, closely matches the true low-rank component  $\bX$ and achieves an RMSE of $7.2 \times 10^{-9}$,  which demonstrates that the proposed method maintains high precision when recovering the underlying low-rank structure.
In addition to the low-rank matrix, the sparse component $\bS_{\text{true}}$ is accurately separated from the observed data, as shown in Figure \ref{F_5}, with an RMSE of $1.1 \times 10^{-8}$. 
The weight matrix ${\bW_k}$ is crucial in optimizing the separation process. 
As shown in Figure \ref{F_6}, during the separation procedure, updates to the weight matrix enable it to accurately identify zero-valued locations within the sparse component and assign them a weight of 1, thereby ensuring precise separation of the sparse component $\bS$.
%Experimental results indicate that, with iterative processing, the proposed method not only maintains accurate recovery of the low-rank component but also effectively extracts the sparse component. This weighting mechanism suppresses noise while avoiding over-estimation of non-sparse regions, significantly enhancing both the accuracy and robustness of the decomposition. Consequently, the dynamic updating of the weight matrix $\hat{\bW_k}$ plays a pivotal role in the separation performance of the model. 
In conclusion, the low RMSE values across both components confirm the robustness of the proposed method in effectively handling noisy and sparse data. The inclusion of the weight matrix further enhances the accuracy of the decomposition by dynamically adjusting to the structure of the data.

To further observe the changes in the weight matrix, we present its evolution over iterations in Figure \ref{Wkk}. 
As shown in the figure, the iterative updates of the weight matrix  $\bW$ reveal how the model progressively refines its focus on essential structural elements in sparse components. 
Initially, the matrix is uniformly weighted (Iter = 1), indicating an absence of prior knowledge regarding sparse component positions. The weight matrix gradually adapts as iterations increase, highlighting significant structural components (the ``M'' shape) while suppressing less relevant features. By Iter = 500, the model has effectively learned to focus on the main structure, with clearer differentiation between significant (dark blue) and negligible (yellow) regions. 
{This process can be analogized to the self-attention mechanism, where the model dynamically assigns different levels of ``attention'' or ``importance'' to each position based on their relevance. In self-attention, each element in a sequence learns to selectively attend to others based on similarity scores, adjusting its focus iteratively. Similarly, in the proposed model,  the weights are adjusted to emphasize important structural elements of the sparse component while reducing sensitivity to noise. This adaptive weighting scheme enables the model to identify and retain critical features with increased robustness.}

\begin{figure*}
\centering
%\hspace{-2.2cm}
\subfloat[Frames]{ 
\begin{minipage}[t]{0.18\textwidth}
\centering
\includegraphics[width=0.6\linewidth]{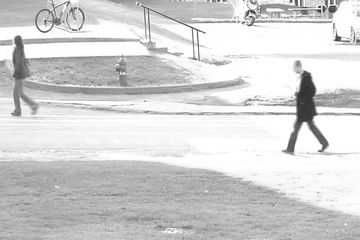} \\ \vspace{14.65mm} 
\includegraphics[width=0.6\linewidth]{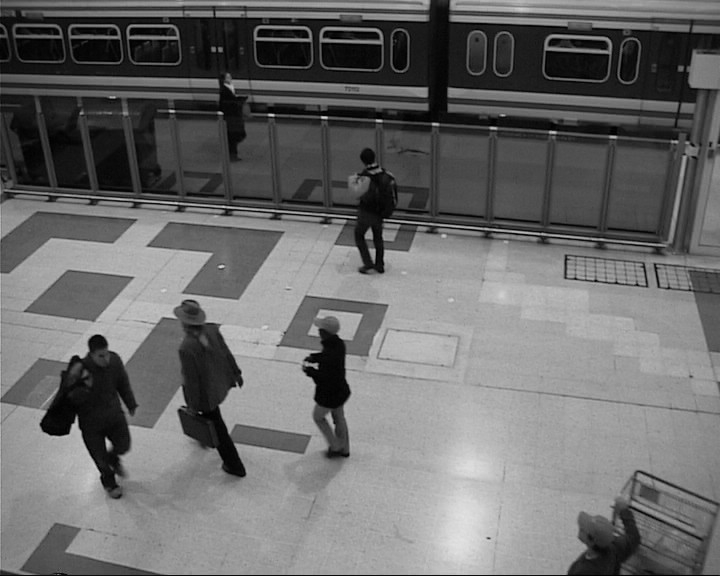} \\
\vspace{17mm}
\includegraphics[width=0.6\linewidth]{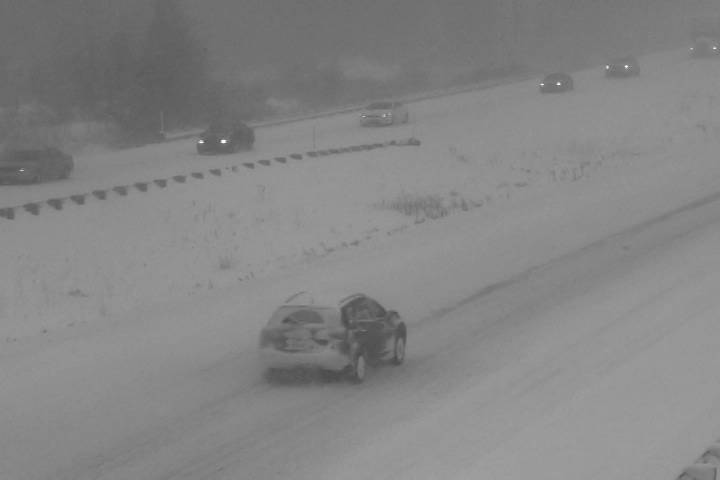} \\
\vspace{13.8mm}
\includegraphics[width=0.6\linewidth]{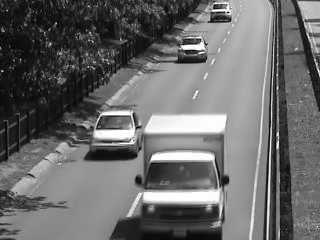} \\
\vspace{16mm}
\includegraphics[width=0.6\linewidth]{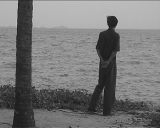} \\
\vspace{15mm}
\end{minipage}} \hspace{-1.4cm}
\subfloat[WNNM]{ 
\begin{minipage}[t]{0.18\textwidth}
\centering
\includegraphics[width=0.6\linewidth]{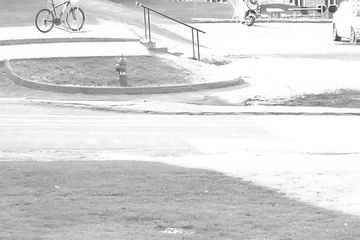}\\ \vspace{1mm} 
\includegraphics[width=0.6\linewidth]{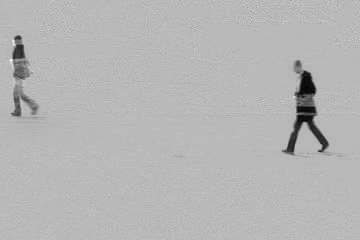} \\ \vspace{1mm}  
\includegraphics[width=0.6\linewidth]{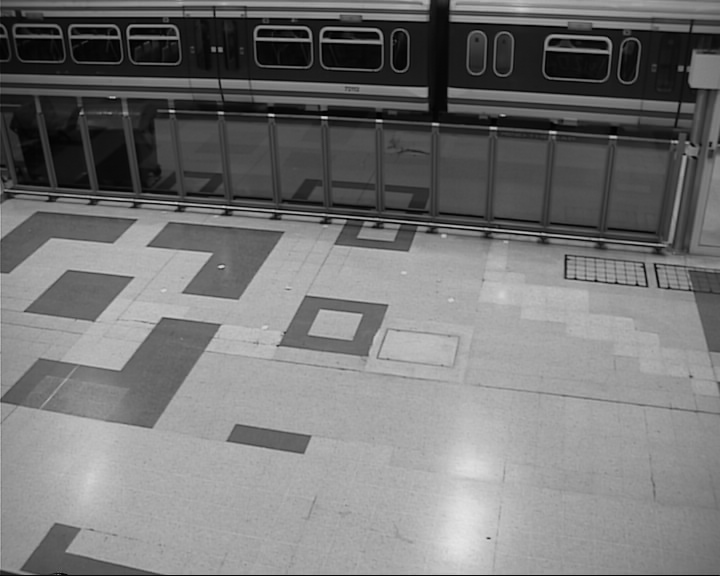}\\ \vspace{1mm} 
\includegraphics[width=0.6\linewidth]{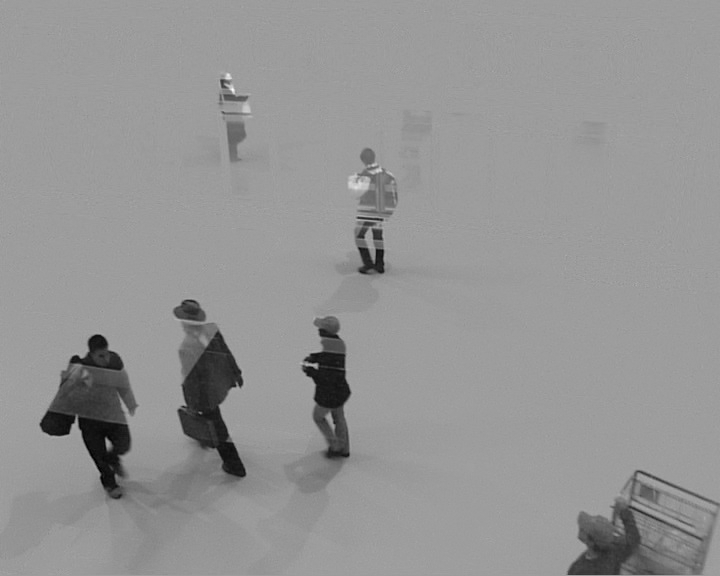} \\ \vspace{1mm}  
\includegraphics[width=0.6\linewidth]{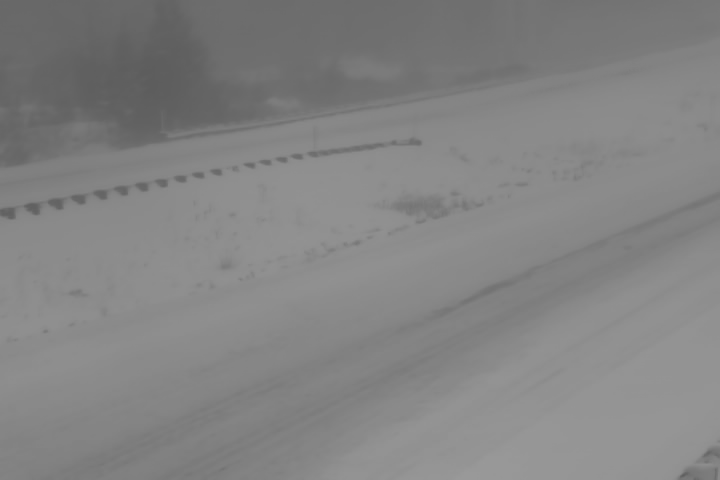}\\ \vspace{1mm} 
\includegraphics[width=0.6\linewidth]{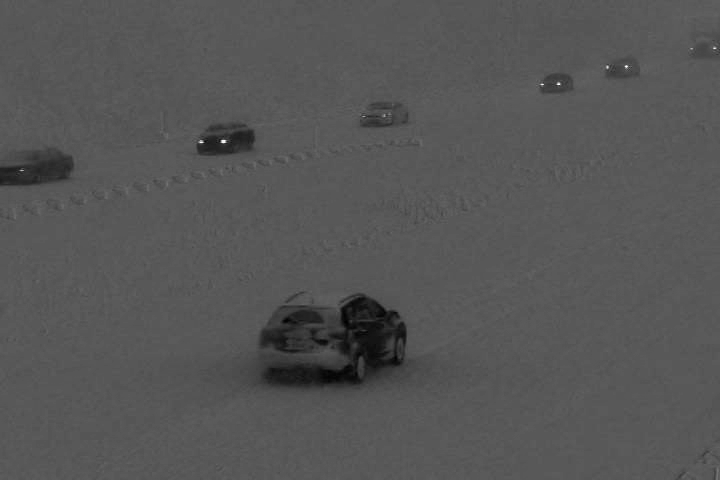} \\ \vspace{1mm}  
\includegraphics[width=0.6\linewidth]{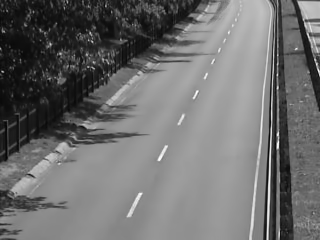}\\ 
\vspace{1mm} 
\includegraphics[width=0.6\linewidth]{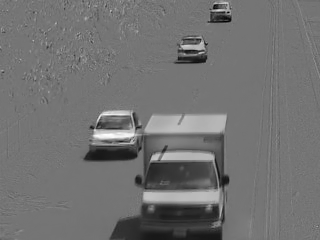} \\ \vspace{1mm}   
\includegraphics[width=0.6\linewidth]{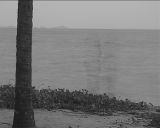}\\ \vspace{1mm} 
\includegraphics[width=0.6\linewidth]{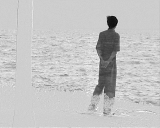} \\
\end{minipage}} \hspace{-1.4cm}
   \subfloat[NC]{ 
\begin{minipage}[t]{0.18\textwidth}
\centering
\includegraphics[width=0.6\linewidth]{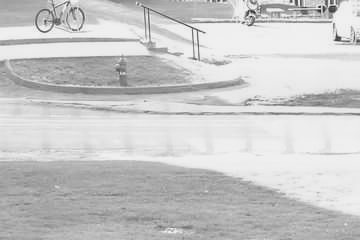}\\ \vspace{1mm} 
\includegraphics[width=0.6\linewidth]{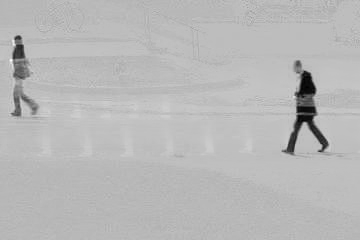}\\  \vspace{1mm}  
\includegraphics[width=0.6\linewidth]{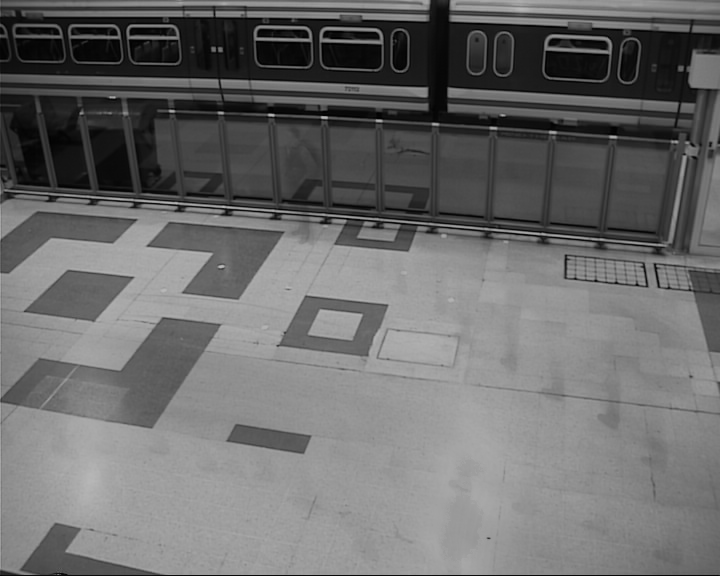}\\ \vspace{1mm} 
\includegraphics[width=0.6\linewidth]{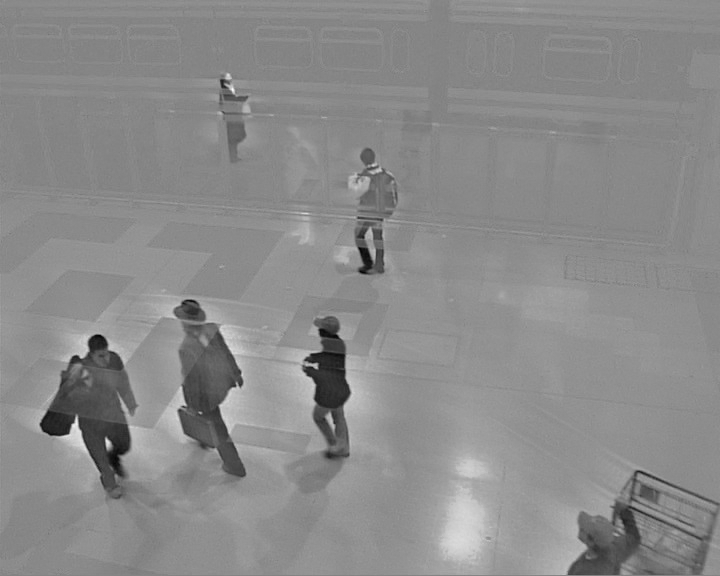} \\ \vspace{1mm}
\includegraphics[width=0.6\linewidth]{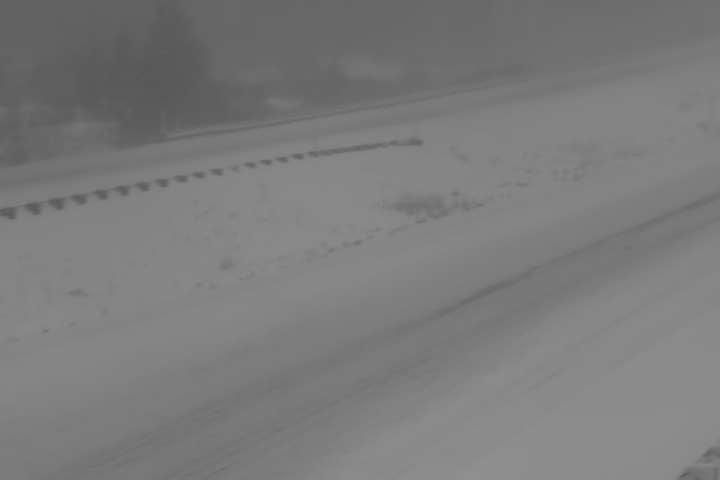}\\ \vspace{1mm} 
\includegraphics[width=0.6\linewidth]{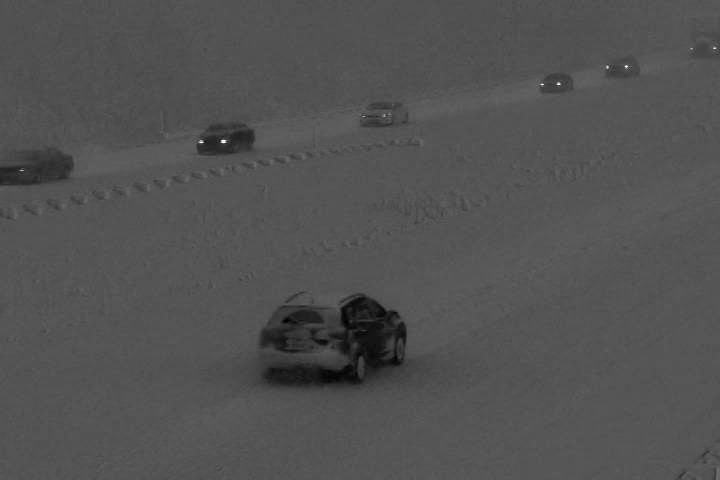} \\ \vspace{1mm}   
\includegraphics[width=0.6\linewidth]{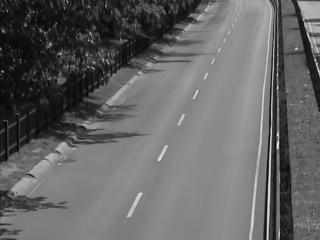}\\ \vspace{1mm} 
\includegraphics[width=0.6\linewidth]{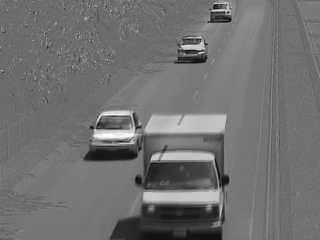} \\ \vspace{1mm}  
\includegraphics[width=0.6\linewidth]{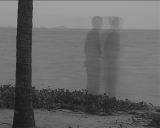}\\ \vspace{1mm} 
\includegraphics[width=0.6\linewidth]{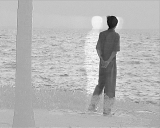} \\
\end{minipage}}  \hspace{-1.4cm}
\subfloat[HQF]{ 
\begin{minipage}[t]{0.18\textwidth}
\centering
\includegraphics[width=0.6\linewidth]{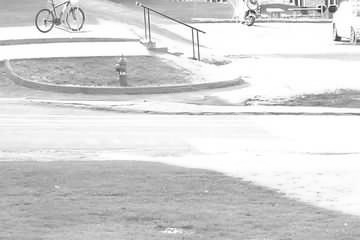}\\ \vspace{1mm} 
\includegraphics[width=0.6\linewidth]{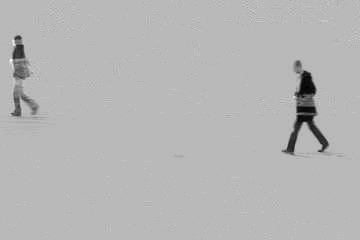}\\  \vspace{1mm}  
\includegraphics[width=0.6\linewidth]{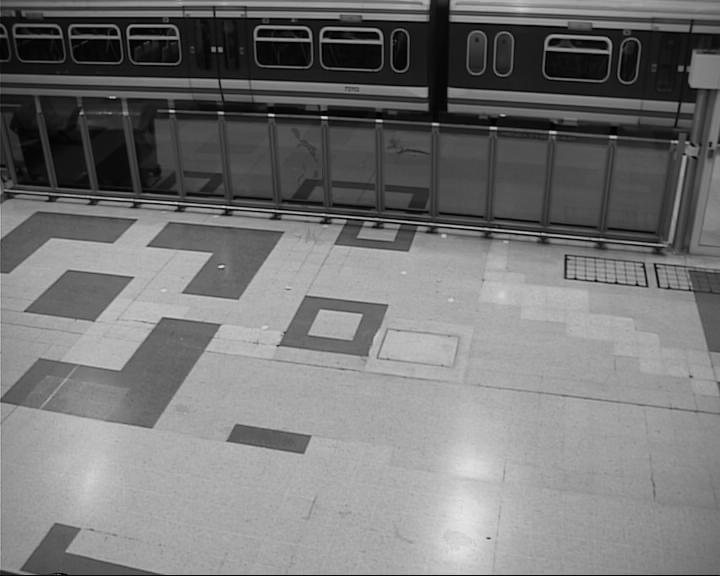}\\ \vspace{1mm} 
\includegraphics[width=0.6\linewidth]{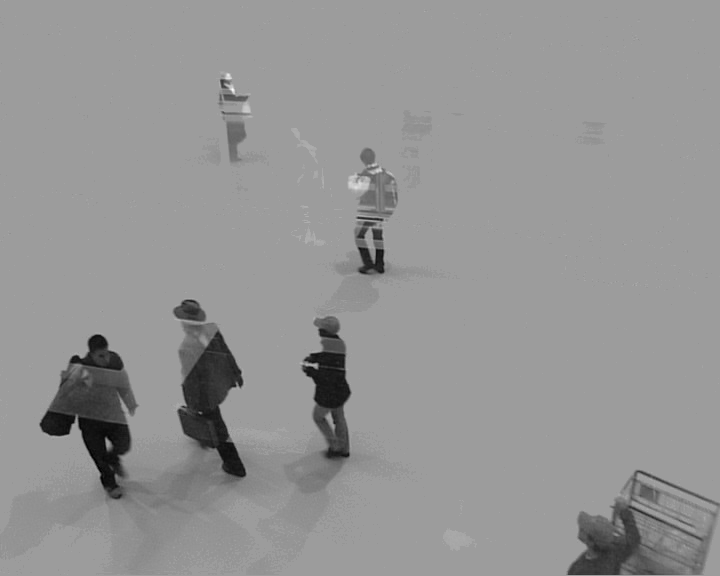} \\ \vspace{1mm}
\includegraphics[width=0.6\linewidth]{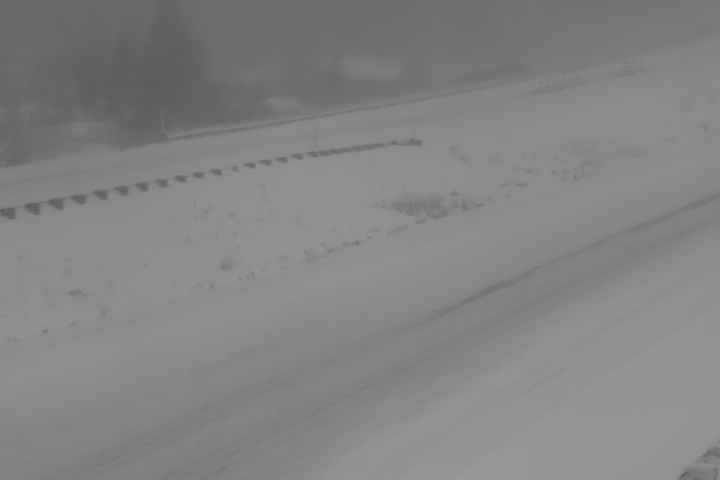}\\ \vspace{1mm} 
\includegraphics[width=0.6\linewidth]{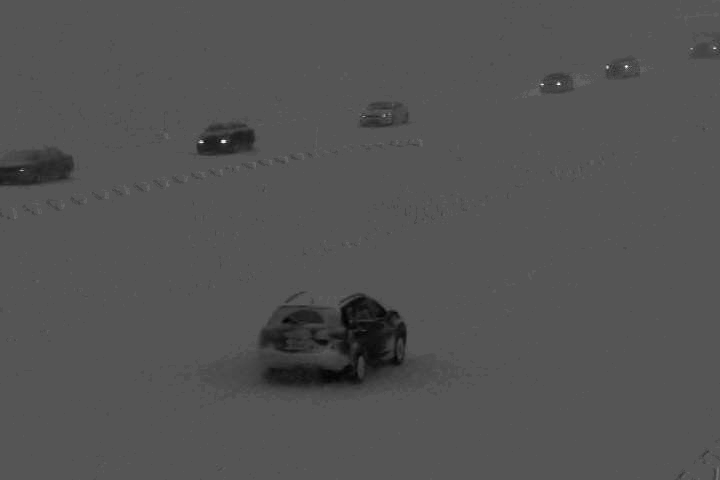} \\ \vspace{1mm}   
\includegraphics[width=0.6\linewidth]{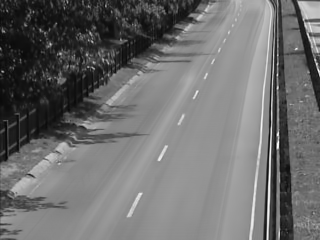}\\ \vspace{1mm} 
\includegraphics[width=0.6\linewidth]{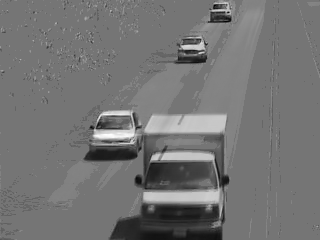} \\ \vspace{1mm}  
\includegraphics[width=0.6\linewidth]{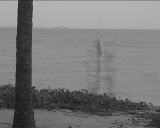}\\ \vspace{1mm} 
\includegraphics[width=0.6\linewidth]{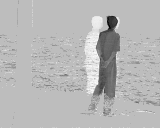} \\
\end{minipage}}  \hspace{-1.4cm}
 \subfloat[OBC]{
\begin{minipage}[t]{0.18\textwidth}
\centering
\includegraphics[width=0.6\linewidth]{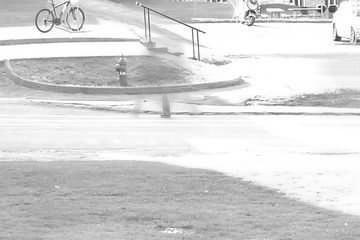}\\ \vspace{1mm} 
\includegraphics[width=0.6\linewidth]{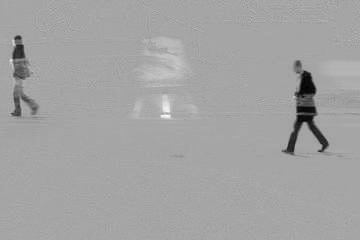} \\ \vspace{1mm}  
\includegraphics[width=0.6\linewidth]{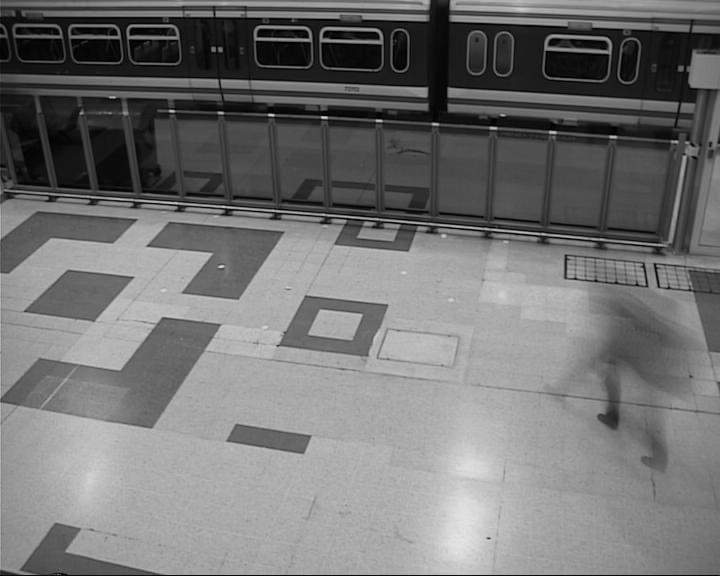}\\ \vspace{1mm} 
\includegraphics[width=0.6\linewidth]{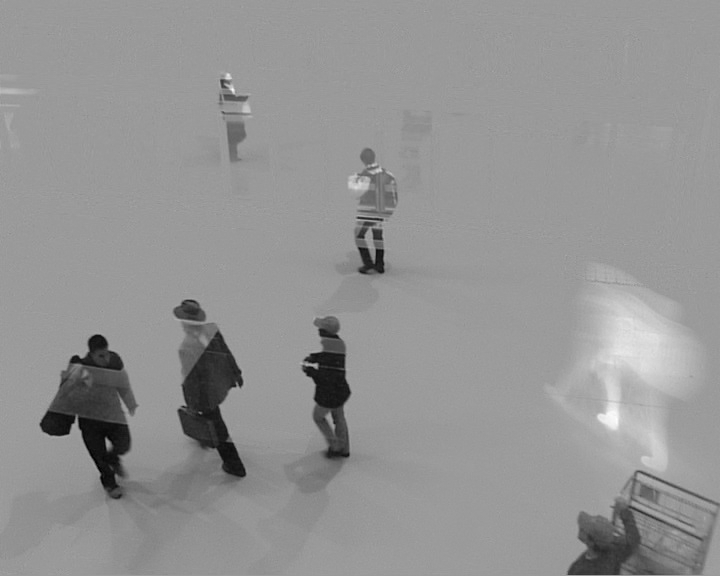} \\ \vspace{1mm}  
\includegraphics[width=0.6\linewidth]{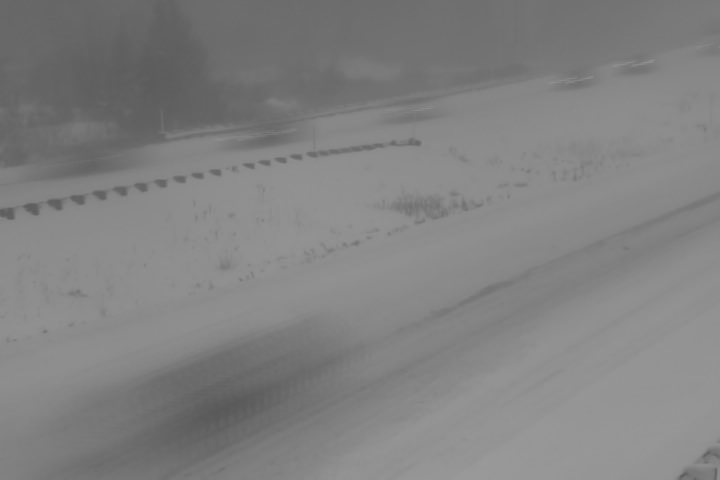}\\ \vspace{1mm} 
\includegraphics[width=0.6\linewidth]{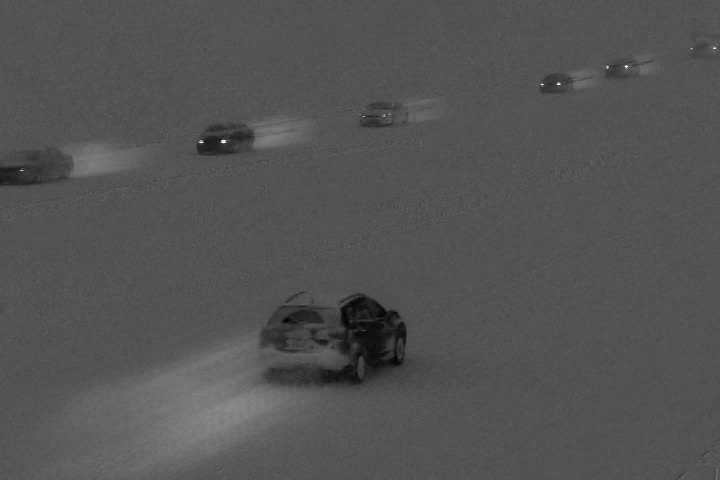} \\ \vspace{1mm}  
\includegraphics[width=0.6\linewidth]{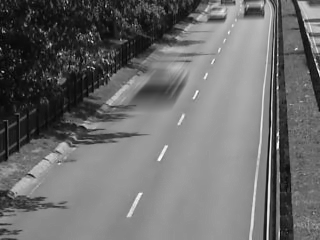}\\ \vspace{1mm} 
\includegraphics[width=0.6\linewidth]{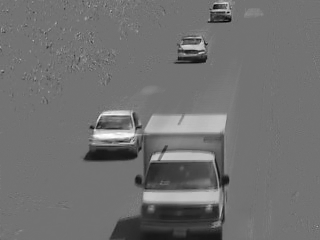} \\ \vspace{1mm} 
\includegraphics[width=0.6\linewidth]{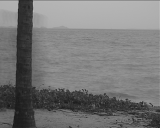}\\ \vspace{1mm} 
\includegraphics[width=0.6\linewidth]{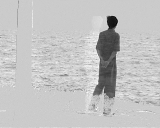} \\
\end{minipage}}\hspace{-1.4cm}
\subfloat[HOW]{ 
\begin{minipage}[t]{0.18\textwidth}
\centering
\includegraphics[width=0.6\linewidth]{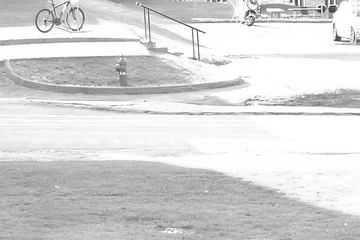}\\ \vspace{1mm} 
\includegraphics[width=0.6\linewidth]{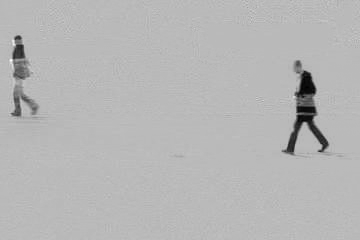} \\ \vspace{1mm}  
\includegraphics[width=0.6\linewidth]{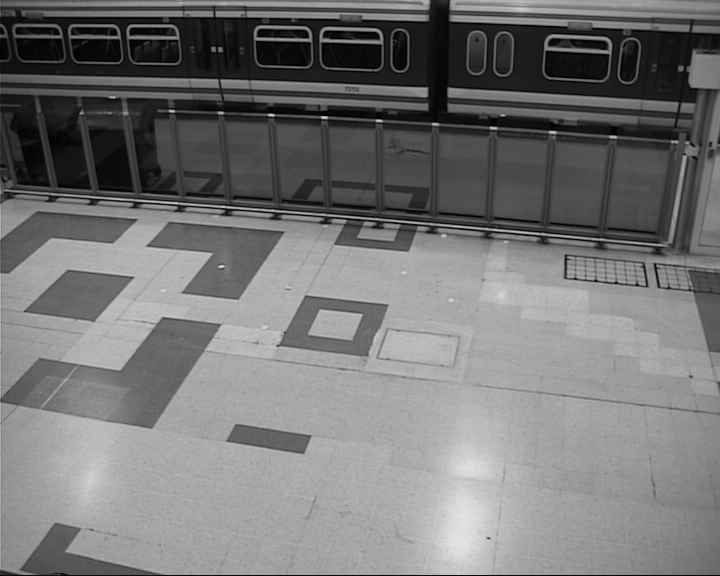}\\ \vspace{1mm} 
\includegraphics[width=0.6\linewidth]{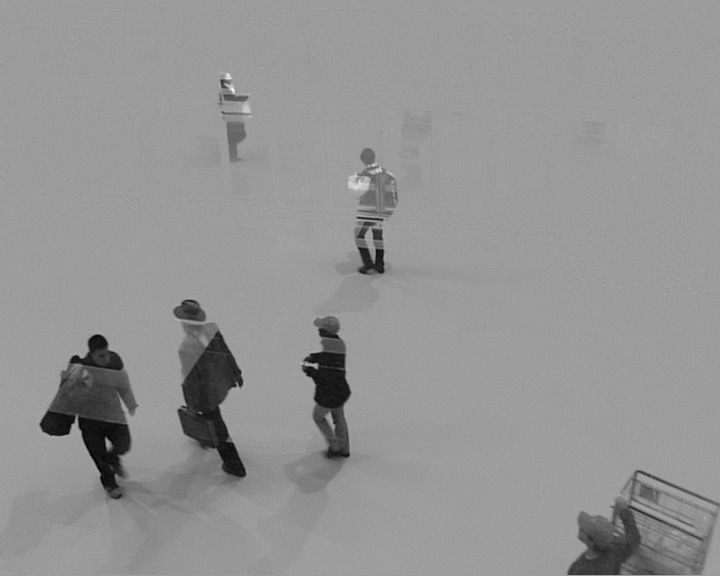} \\ \vspace{1mm}  
\includegraphics[width=0.6\linewidth]{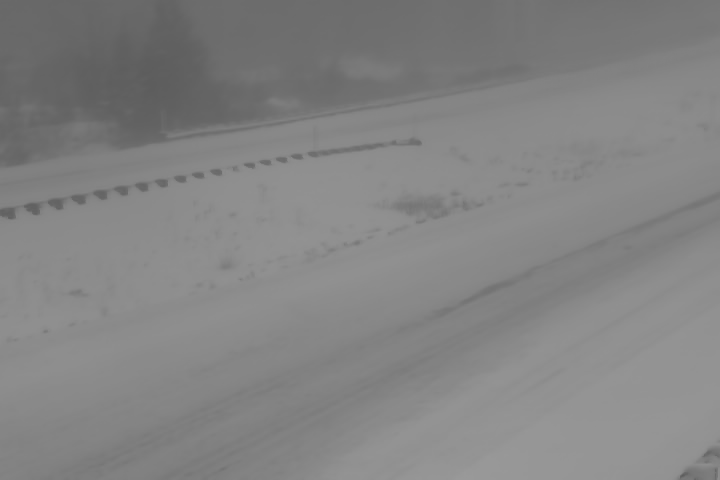}\\ \vspace{1mm} 
\includegraphics[width=0.6\linewidth]{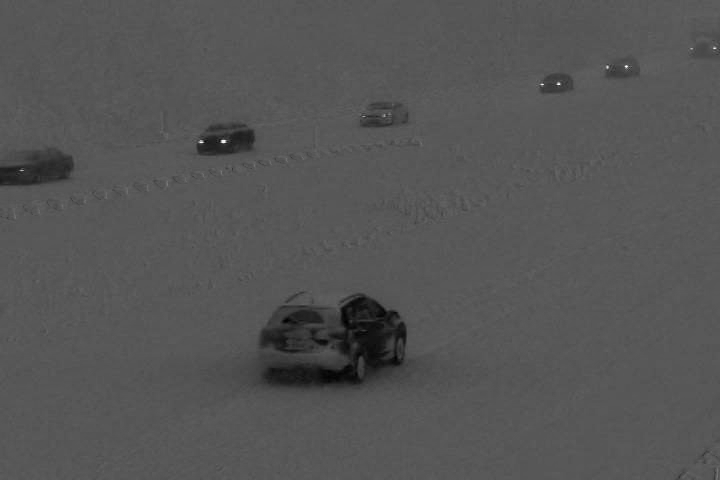} \\ \vspace{1mm}  
\includegraphics[width=0.6\linewidth]{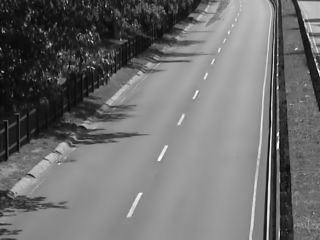}\\ \vspace{1mm} 
\includegraphics[width=0.6\linewidth]{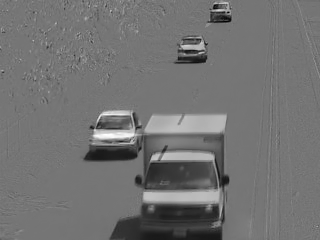} \\ \vspace{1mm} 
\includegraphics[width=0.6\linewidth]{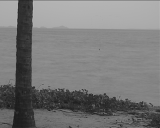}\\ \vspace{1mm} 
\includegraphics[width=0.6\linewidth]{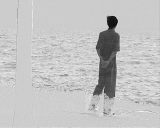} \\
\end{minipage}}\hspace{-1.4cm}
\subfloat[W-L2]{ 
\begin{minipage}[t]{0.18\textwidth}
\centering
\includegraphics[width=0.6\linewidth]{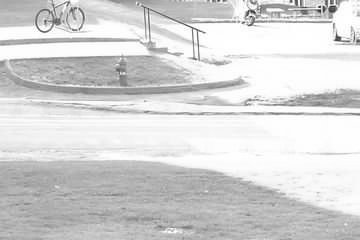}\\ \vspace{1mm} 
\includegraphics[width=0.6\linewidth]{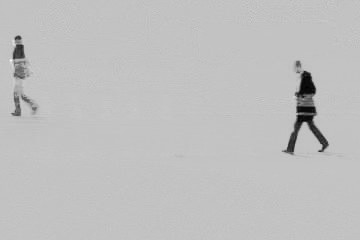} \\ \vspace{1mm}  
\includegraphics[width=0.6\linewidth]{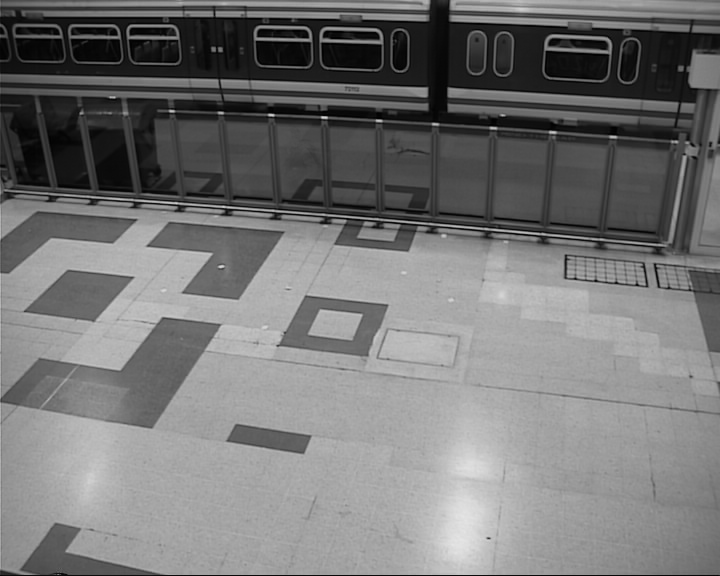}\\ \vspace{1mm} 
\includegraphics[width=0.6\linewidth]{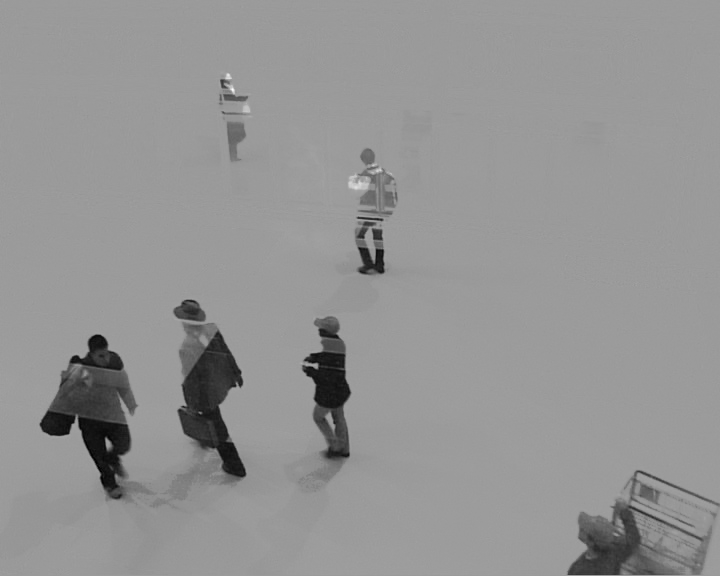} \\ \vspace{1mm}  
\includegraphics[width=0.6\linewidth]{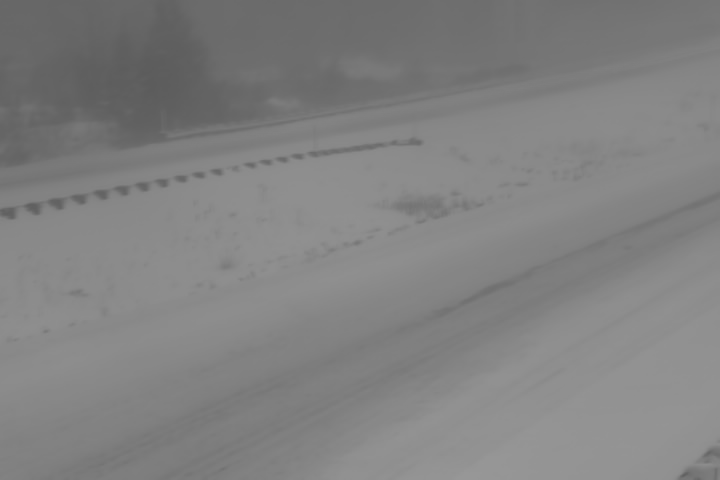}\\ \vspace{1mm} 
\includegraphics[width=0.6\linewidth]{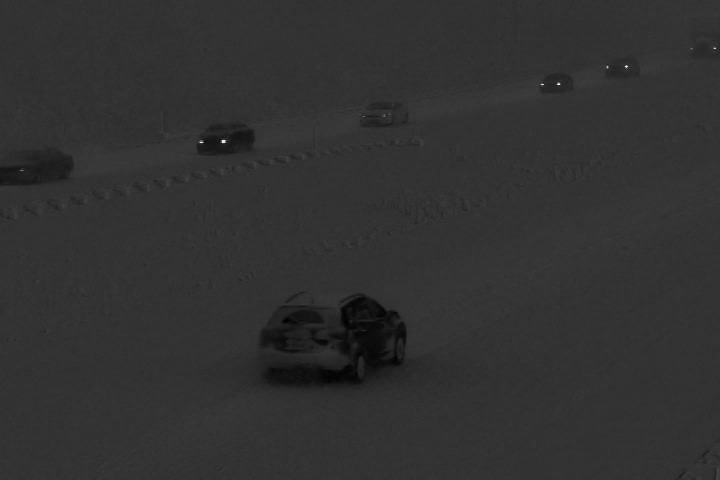} \\ \vspace{1mm}  
\includegraphics[width=0.6\linewidth]{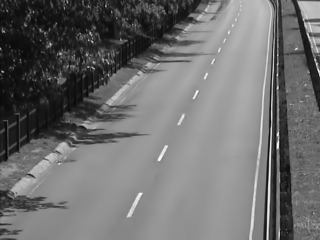}\\ \vspace{1mm} 
\includegraphics[width=0.6\linewidth]{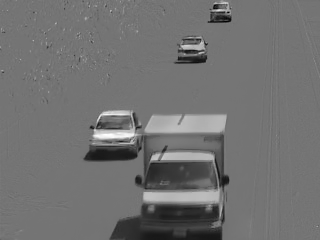} \\ \vspace{1mm} 
\includegraphics[width=0.6\linewidth]{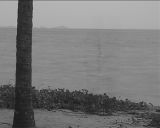}\\ \vspace{1mm} 
\includegraphics[width=0.6\linewidth]{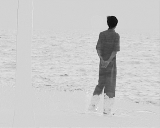} \\
\end{minipage}}%\hspace{-1.05cm}
\caption{Background subtraction for five 200-frame surveillance video sequences. From top to bottom: Pedestrians, PETS2006, Blizzard, Highway, and WaterSurface, respectively.   }
\label{Video}
\end{figure*}

\begin{figure*}
\centering
%\hspace{-0.5cm}
\subfloat[Observed]{ 
\begin{minipage}[t]{0.16\textwidth}
\centering
\includegraphics[width=0.6\linewidth]{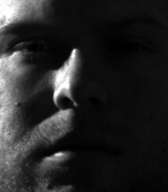} \\ \vspace{1mm} 
\includegraphics[width=0.6\linewidth]{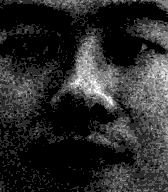} \\
\vspace{1mm} 
\includegraphics[width=0.6\linewidth]{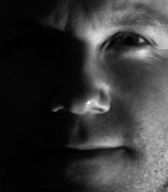} \\ \vspace{1mm} 
\includegraphics[width=0.6\linewidth]{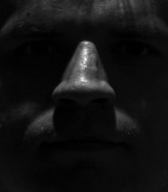} \\  \vspace{1mm} 
\includegraphics[width=0.6\linewidth]{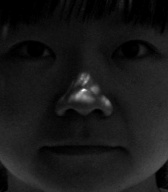} \\ \vspace{1mm} 
\includegraphics[width=0.6\linewidth]{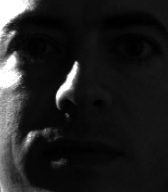} \\
\vspace{1mm} 
\includegraphics[width=0.6\linewidth]{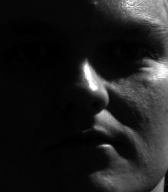} \\ \vspace{1mm} 
\includegraphics[width=0.6\linewidth]{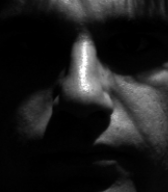} \\  \vspace{1mm} 
\end{minipage}}
\hspace{-1.2cm}
\subfloat[WNNM]{ 
\begin{minipage}[t]{0.16\textwidth}
\centering
\includegraphics[width=0.6\linewidth]{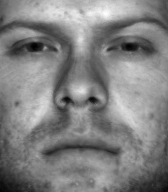} \\ \vspace{1mm} 
\includegraphics[width=0.6\linewidth]{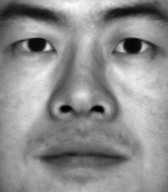} \\
\vspace{1mm} 
\includegraphics[width=0.6\linewidth]{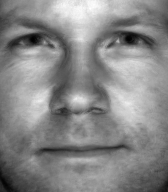} \\ \vspace{1mm} 
\includegraphics[width=0.6\linewidth]{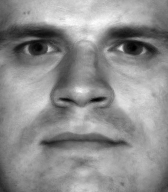} \\  \vspace{1mm} 
\includegraphics[width=0.6\linewidth]{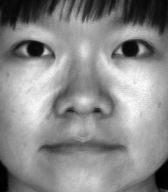} \\ \vspace{1mm} 
\includegraphics[width=0.6\linewidth]{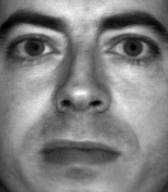} \\
\vspace{1mm} 
\includegraphics[width=0.6\linewidth]{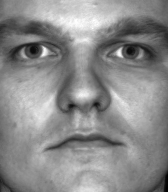} \\ \vspace{1mm} 
\includegraphics[width=0.6\linewidth]{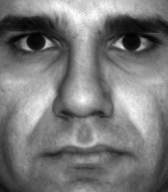} \\  \vspace{1mm} 
\end{minipage}} \hspace{-1.2cm}
\subfloat[NC]{ 
\begin{minipage}[t]{0.16\textwidth}
\centering
\includegraphics[width=0.6\linewidth]{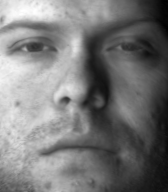} \\ \vspace{1mm} 
\includegraphics[width=0.6\linewidth]{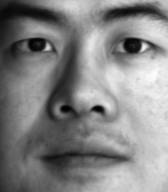} \\
\vspace{1mm} 
\includegraphics[width=0.6\linewidth]{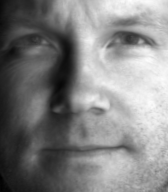} \\ \vspace{1mm} 
\includegraphics[width=0.6\linewidth]{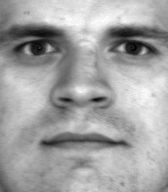} \\  \vspace{1mm} 
\includegraphics[width=0.6\linewidth]{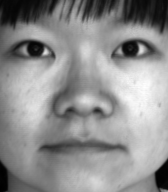} \\ \vspace{1mm} 
\includegraphics[width=0.6\linewidth]{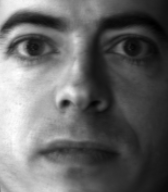} \\
\vspace{1mm} 
\includegraphics[width=0.6\linewidth]{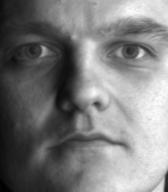} \\ \vspace{1mm} 
\includegraphics[width=0.6\linewidth]{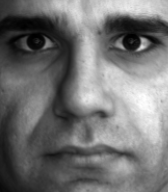} \\  \vspace{1mm} 
\end{minipage}} \hspace{-1.2cm}
\subfloat[HQF]{ 
\begin{minipage}[t]{0.16\textwidth}
\centering
\includegraphics[width=0.6\linewidth]{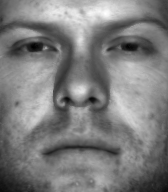} \\ \vspace{1mm} 
\includegraphics[width=0.6\linewidth]{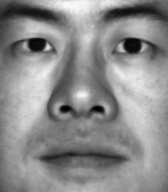} \\
\vspace{1mm} 
\includegraphics[width=0.6\linewidth]{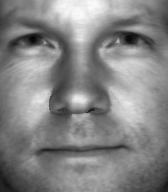} \\ \vspace{1mm} 
\includegraphics[width=0.6\linewidth]{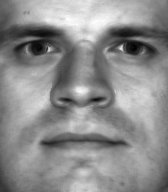} \\  \vspace{1mm} 
\includegraphics[width=0.6\linewidth]{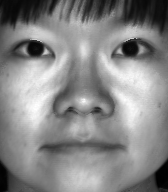} \\ \vspace{1mm} 
\includegraphics[width=0.6\linewidth]{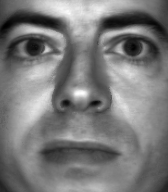} \\
\vspace{1mm} 
\includegraphics[width=0.6\linewidth]{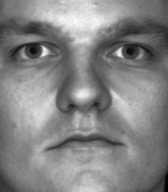} \\ \vspace{1mm} 
\includegraphics[width=0.6\linewidth]{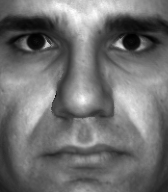} \\  \vspace{1mm} 
\end{minipage}} \hspace{-1.2cm}
\subfloat[OBC]{ 
\begin{minipage}[t]{0.16\textwidth}
\centering
\includegraphics[width=0.6\linewidth]{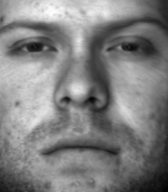} \\ \vspace{1mm} 
\includegraphics[width=0.6\linewidth]{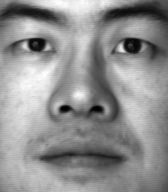} \\
\vspace{1mm} 
\includegraphics[width=0.6\linewidth]{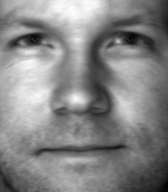} \\ \vspace{1mm} 
\includegraphics[width=0.6\linewidth]{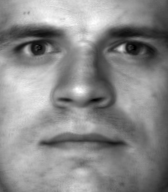} \\  \vspace{1mm} 
\includegraphics[width=0.6\linewidth]{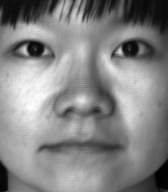} \\ \vspace{1mm} 
\includegraphics[width=0.6\linewidth]{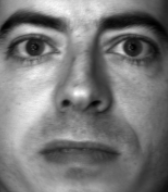} \\
\vspace{1mm} 
\includegraphics[width=0.6\linewidth]{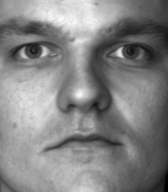} \\ \vspace{1mm} 
\includegraphics[width=0.6\linewidth]{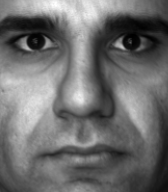} \\  \vspace{1mm} 
\end{minipage}} \hspace{-1.2cm}
\subfloat[HOW]{ 
\begin{minipage}[t]{0.16\textwidth}
\centering
\includegraphics[width=0.6\linewidth]{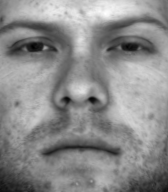} \\ \vspace{1mm} 
\includegraphics[width=0.6\linewidth]{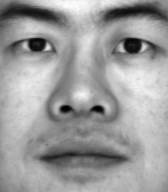} \\
\vspace{1mm} 
\includegraphics[width=0.6\linewidth]{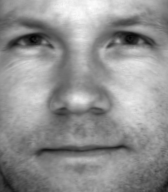} \\ \vspace{1mm} 
\includegraphics[width=0.6\linewidth]{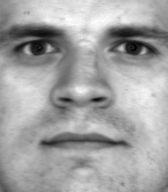} \\  \vspace{1mm} 
\includegraphics[width=0.6\linewidth]{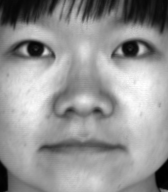} \\ \vspace{1mm} 
\includegraphics[width=0.6\linewidth]{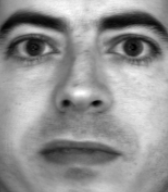} \\
\vspace{1mm} 
\includegraphics[width=0.6\linewidth]{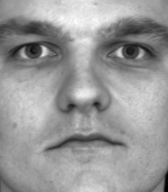} \\ \vspace{1mm} 
\includegraphics[width=0.6\linewidth]{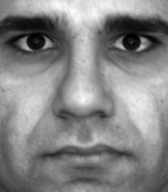} \\  \vspace{1mm} 
\end{minipage}} \hspace{-1.2cm}
\subfloat[W-L2]{ 
\begin{minipage}[t]{0.16\textwidth}
\centering
\includegraphics[width=0.6\linewidth]{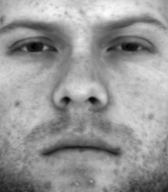} \\ \vspace{1mm} 
\includegraphics[width=0.6\linewidth]{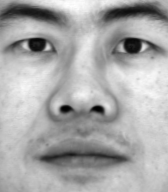} \\
\vspace{1mm} 
\includegraphics[width=0.6\linewidth]{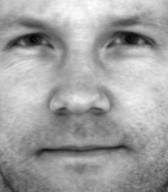} \\ \vspace{1mm} 
\includegraphics[width=0.6\linewidth]{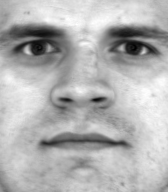} \\  \vspace{1mm} 
\includegraphics[width=0.6\linewidth]{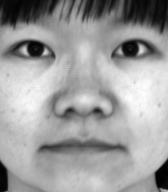} \\ \vspace{1mm} 
\includegraphics[width=0.6\linewidth]{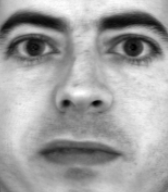} \\
\vspace{1mm} 
\includegraphics[width=0.6\linewidth]{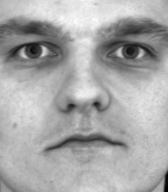} \\ \vspace{1mm} 
\includegraphics[width=0.6\linewidth]{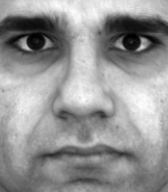} \\  \vspace{1mm} 
\end{minipage}} \hspace{-1.2cm}
\caption{Face shadow removal for Yale B dataset. }
\label{YaleB}
\end{figure*}

\subsection{Comparison with the state-of-art  methods} 
To assess the performance of the proposed method, we conducted numerical experiments and compared the proposed method, called W-L2, with several state-of-the-art methods, including GoDec+ \cite{guo2017godec+}, WNNM \cite{gu2017weighted}, NC \cite{wen2019robust}, OBC \cite{liu2021efficient}, HQF \cite{wang2023robustHQF}, and HOW \cite{wang2023robust}. It is worth noting that in model \ref{RPCA-WLS}, the $F$ norm can be replaced with the $\ell_0$ norm, introducing an adaptive weighted $\ell_0$ norm. In the first set of experiments, we also compare the results of the W-L0 method with those methods.

\subsubsection{Synthetic experiments}

We first conduct a comparison of synthetic data.  
We construct random matrices $ \bU \in \mathbb{R}^{m \times r} $ and $ \bV \in \mathbb{R}^{r \times n} $ to generate the synthetic matrix $ \bX = \bU \bV $, where each element is independently drawn from a standard Gaussian distribution. For simplicity, we set $ m = n $ and define the rank as $ r = m/50 $, ensuring that $ \bX $ exhibits a low-rank structure. 
Next, we add the sparse noise $ \bS $ to $ \bX $ to obtain the observation matrix. This sparse noise $ \bS $ is designed with varying levels of sparsity and SNR to evaluate the robustness of each method under different sparsity and noise conditions.
Finally, we perform a quantitative comparison of various methods to assess their performance in low-rank and sparse decomposition.

Tables \ref{Table1} and \ref{Table2} show the RMSE of the low-rank matrix $\bX$ separated using different methods at the sparsity levels of 10\% and 20\% under varying SNRs, respectively. The experiments are conducted for two matrix sizes, $m = 500$ and $m = 1000$. The RMSE values and the average computation time for each method are provided for both matrix sizes. 
In these experiments, we observed that methods such as W-L2 and W-L0 exhibited excellent performance in terms of RMSE across various scenarios. Specifically, the proposed W-L2 method achieved the lowest RMSE values across all tested conditions, consistently remaining at the extremely low order of magnitude of $10^{-11}$.  Although the W-L0 method has higher error than W-L2, it offers a significant advantage in computational efficiency, making it more suitable for large-scale or real-time applications.  
This result demonstrates that weight-based regularization techniques provide significant performance improvements in RPCA. 
In addition to accuracy, computational efficiency is a key factor in practical applications.  

We noted that methods like W-L2, HQF and HOW exhibit relatively short computation times compared to traditional methods such as WNNM and NC. Notably, the NC method took 34.03 seconds when $m = 500$ and as long as 159.53 seconds when $m = 1000$. This indicates that, although these traditional methods show reasonably accurate results in terms of RMSE,  their high computational costs limit their potential application in large-scale problems.
In summary, the experimental data in Tables \ref{Table1} and \ref{Table2} demonstrate that the proposed W-L2 method not only excels in terms of accuracy but also maintains competitive computational efficiency.

\subsubsection{Background subtraction from video}
In this subsection, we use RPCA to model surveillance video data, in order to effectively separate the dynamic foreground (sparse component) from the static background (low-rank component) in the videos. We selected five benchmark video sequences, each containing 200 frames, from the CDnet 2014 dataset \cite{wang2014cdnet} and \cite{li2004statistical}.
To perform matrix decomposition, we first reshape the three-dimensional video data into a matrix of size $ mn \times 200 $. Given that the background remains nearly static, we decompose the matrix $ \bY $ into the sum of a low-rank matrix $ \bX $ with rank 1 and a sparse matrix $ \bS $. Subsequently, we rearrange the decomposed matrices $ \bX $ and $ \bS $ by columns to form a data structure of size $ m \times n \times 200 $. This decomposition process enables us to effectively distinguish the background and foreground in the videos.

Figure \ref{Video} provides a detailed comparison of the background subtraction performance of different methods across five 200-frame surveillance video sequences. The results demonstrate that the W-L2 method consistently excels across all sequences, offering not only clear and precise foreground extraction but also significantly reducing noise and ghosting effects. Notably, the W-L2 method maintains object boundary clarity and accuracy even when processing sequences with subtle movements or small objects. The HOW method also performs well, though slight background residuals can be observed in more complex scenes. In contrast, the HQF, OBC  and NC methods achieve a certain degree of foreground separation but often introduce noticeable background noise, which directly affects the accuracy of foreground extraction. In summary, the W-L2 method stands out for its superior performance in foreground extraction accuracy and robustness to environmental changes, while also maintaining efficient computation.

\subsubsection{Face shadow removal}
Additionally, we conducted experiments on the facial shadow removal task to evaluate the effectiveness of our proposed method for different tasks. We used the Yale B dataset, which contains 64 facial images of size $ 168 \times 192 $ under various lighting conditions and angles. In this task, we applied a low-rank and sparse decomposition strategy to handle shadow removal: the low-rank component represents relatively constant background information, such as facial contours and overall shape, while the sparse component corresponds to localized variations in shadows and lighting artifacts. 
To perform matrix decomposition, we vectorized the 64 images and constructed a matrix $ \bY $ of size $ 32256 \times 64 $. We then decomposed $ \bY $ into the sum of a low-rank matrix $ \bX $ with rank 1 and a sparse matrix $ \bS $. 

The restoration results for different methods are shown in Figure \ref{YaleB}.  The experimental results demonstrate that the W-L2 method consistently achieves excellent shadow removal performance across all tested face images. This method not only effectively eliminates shadow artifacts, but also successfully preserves facial details and textures, producing natural and consistent image outputs. 
In comparison, the HOW method also delivers good shadow removal results; however, in some cases, it leads to blurred facial features, which affects the overall clarity of the images. HQF and OBC methods provide relatively accurate results, but they tend to leave residual shadow artifacts, especially in regions with strong lighting contrasts, thus compromising the visual quality of the images.
On the other hand, the WNNM and NC methods exhibit the most significant shortcomings, struggling to retain facial details while removing shadows, which results in overly smoothed images and loss of critical features. 
In summary, the W-L2 method shows superior performance in face shadow removal, maintaining both structural integrity and visual quality.

\section{Conclusion} \label{Con}
In this paper, we proposed a novel approach to RPCA that addresses the limitations of traditional methods. Our model utilizes a weighted Frobenius norm to represent sparse components,  our model reduces bias and simplifies the optimization process compared to the conventional $\ell_1$-norm. We employed an alternating minimization algorithm, ensuring that each subproblem has an explicit solution, which improves computational efficiency.
Despite its simplicity, numerical experiments demonstrated that our method outperforms existing non-convex regularization techniques, showing improved accuracy and robustness in practical applications.

%\section*{Acknowledgments}

%\section{References Section}

%\bibliography{IEEEabrv, Ref}

\begin{thebibliography}{10}

\bibitem{abdi2010principal}
H.~Abdi and L.~J. Williams, ``Principal component analysis,'' \emph{Wiley Interdisciplinary Reviews: Computational Statistics}, vol.~2, no.~4, pp. 433--459, 2010.

\bibitem{jolliffe2016principal}
I.~T. Jolliffe and J.~Cadima, ``Principal component analysis: a review and recent developments,'' \emph{Philosophical Transactions of the Royal Society A: Mathematical, Physical and Engineering Sciences}, vol. 374, no. 2065, p. 20150202, 2016.

\bibitem{xu2010robust}
H.~Xu, C.~Caramanis, and S.~Sanghavi, ``{Robust PCA via outlier pursuit},'' \emph{Advances in Neural Information Processing Systems}, vol.~23, 2010.

\bibitem{nie2020truncated}
F.~Nie, D.~Wu, R.~Wang, and X.~Li, ``Truncated robust principle component analysis with a general optimization framework,'' \emph{IEEE Transactions on Pattern Analysis and Machine Intelligence}, vol.~44, no.~2, pp. 1081--1097, 2020.

\bibitem{bouwmans2018applications}
T.~Bouwmans, S.~Javed, H.~Zhang, Z.~Lin, and R.~Otazo, ``{On the applications of robust PCA in image and video processing},'' \emph{Proceedings of the IEEE}, vol. 106, no.~8, pp. 1427--1457, 2018.

\bibitem{gu2017weighted}
S.~Gu, Q.~Xie, D.~Meng, W.~Zuo, X.~Feng, and L.~Zhang, ``Weighted nuclear norm minimization and its applications to low level vision,'' \emph{International Journal of Computer Vision}, vol. 121, pp. 183--208, 2017.

\bibitem{wang2021tensor}
Y.~Wang, T.~Li, L.~Chen, Y.~Yu, Y.~Zhao, and J.~Zhou, ``Tensor-based robust principal component analysis with locality preserving graph and frontal slice sparsity for hyperspectral image classification,'' \emph{IEEE Transactions on Geoscience and Remote Sensing}, vol.~60, pp. 1--19, 2021.

\bibitem{jiang2018superpca}
J.~Jiang, J.~Ma, C.~Chen, Z.~Wang, Z.~Cai, and L.~Wang, ``{SuperPCA: A superpixelwise PCA approach for unsupervised feature extraction of hyperspectral imagery},'' \emph{IEEE Transactions on Geoscience and Remote Sensing}, vol.~56, no.~8, pp. 4581--4593, 2018.

\bibitem{liu2021efficient}
Q.~Liu and X.~Li, ``Efficient low-rank matrix factorization based on $\ell_{1,\varepsilon}$-norm for online background subtraction,'' \emph{IEEE Transactions on Circuits and Systems for Video Technology}, vol.~32, no.~7, pp. 4900--4904, 2021.

\bibitem{cao2016total}
W.~Cao, Y.~Wang, J.~Sun, D.~Meng, C.~Yang, A.~Cichocki, and Z.~Xu, ``{Total variation regularized tensor RPCA for background subtraction from compressive measurements},'' \emph{IEEE Transactions on Image Processing}, vol.~25, no.~9, pp. 4075--4090, 2016.

\bibitem{ruhan2022enhance}
A.~Ruhan, X.~Mu, and J.~He, ``{Enhance tensor RPCA-based Mahalanobis distance method for hyperspectral anomaly detection},'' \emph{IEEE Geoscience and Remote Sensing Letters}, vol.~19, pp. 1--5, 2022.

\bibitem{yao2022hyperspectral}
W.~Yao, L.~Li, H.~Ni, W.~Li, and R.~Tao, ``{Hyperspectral anomaly detection based on improved RPCA with non-convex regularization},'' \emph{Remote Sensing}, vol.~14, no.~6, p. 1343, 2022.

\bibitem{xu2018joint}
Y.~Xu, Z.~Wu, J.~Chanussot, and Z.~Wei, ``{Joint reconstruction and anomaly detection from compressive hyperspectral images using Mahalanobis distance-regularized tensor RPCA},'' \emph{IEEE Transactions on Geoscience and Remote Sensing}, vol.~56, no.~5, pp. 2919--2930, 2018.

\bibitem{xiao2023robust}
Q.~Xiao, L.~Zhao, S.~Chen, and X.~Li, ``Robust tensor low-rank sparse representation with saliency prior for hyperspectral anomaly detection,'' \emph{IEEE Transactions on Geoscience and Remote Sensing}, 2023.

\bibitem{candes2011robust}
E.~J. Cand{\`e}s, X.~Li, Y.~Ma, and J.~Wright, ``Robust principal component analysis?'' \emph{Journal of the ACM (JACM)}, vol.~58, no.~3, pp. 1--37, 2011.

\bibitem{recht2010guaranteed}
B.~Recht, M.~Fazel, and P.~A. Parrilo, ``Guaranteed minimum-rank solutions of linear matrix equations via nuclear norm minimization,'' \emph{SIAM Review}, vol.~52, no.~3, pp. 471--501, 2010.

\bibitem{cai2010singular}
J.-F. Cai, E.~J. Cand{\`e}s, and Z.~Shen, ``A singular value thresholding algorithm for matrix completion,'' \emph{SIAM Journal on Optimization}, vol.~20, no.~4, pp. 1956--1982, 2010.

\bibitem{toh2010accelerated}
K.-C. Toh and S.~Yun, ``An accelerated proximal gradient algorithm for nuclear norm regularized linear least squares problems,'' \emph{Pacific Journal of Optimization}, vol.~6, no. 615-640, p.~15, 2010.

\bibitem{li2015accelerated}
H.~Li and Z.~Lin, ``Accelerated proximal gradient methods for nonconvex programming,'' \emph{Advances in Neural Information Processing Systems}, vol.~28, 2015.

\bibitem{Yuan2013SparseAL}
X.~Yuan and J.~Yang, ``Sparse and low-rank matrix decomposition via alternating direction method,'' \emph{Pacific Journal of Optimization}, vol.~9, p. 167, 2013.

\bibitem{lin2011linearized}
Z.~Lin, R.~Liu, and Z.~Su, ``Linearized alternating direction method with adaptive penalty for low-rank representation,'' \emph{Advances in Neural Information Processing Systems}, vol.~24, 2011.

\bibitem{kang2015robust}
Z.~Kang, C.~Peng, and Q.~Cheng, ``{Robust PCA via nonconvex rank approximation},'' in \emph{2015 IEEE International Conference on Data Mining}.\hskip 1em plus 0.5em minus 0.4em\relax IEEE, 2015, pp. 211--220.

\bibitem{oh2015partial}
T.-H. Oh, Y.-W. Tai, J.-C. Bazin, H.~Kim, and I.~S. Kweon, ``{Partial sum minimization of singular values in robust PCA: Algorithm and applications},'' \emph{IEEE Transactions on Pattern Analysis and Machine Intelligence}, vol.~38, no.~4, pp. 744--758, 2015.

\bibitem{zhang2023generalized}
F.~Zhang, H.~Wang, W.~Qin, X.~Zhao, and J.~Wang, ``{Generalized nonconvex regularization for tensor RPCA and its applications in visual inpainting},'' \emph{Applied Intelligence}, vol.~53, no.~20, pp. 23\,124--23\,146, 2023.

\bibitem{zhang2023hyperspectral}
J.~Zhang, J.~Lu, C.~Wang, and S.~Li, ``Hyperspectral and multispectral image fusion via superpixel-based weighted nuclear norm minimization,'' \emph{IEEE Transactions on Geoscience and Remote Sensing}, 2023.

\bibitem{xie2016weighted}
Y.~Xie, S.~Gu, Y.~Liu, W.~Zuo, W.~Zhang, and L.~Zhang, ``Weighted schatten $ p $-norm minimization for image denoising and background subtraction,'' \emph{IEEE Transactions on Image Processing}, vol.~25, no.~10, pp. 4842--4857, 2016.

\bibitem{huang2023robust}
Y.~Huang, Z.~Wang, Q.~Chen, and W.~Chen, ``Robust principal component analysis via truncated $ l_{1-2} $ minimization,'' in \emph{2023 International Joint Conference on Neural Networks (IJCNN)}, 2023, pp. 1--9.

\bibitem{zhou2011godec}
T.~Zhou and D.~Tao, ``Godec: Randomized low-rank \& sparse matrix decomposition in noisy case,'' in \emph{Proceedings of the 28th International Conference on Machine Learning}, 2011.

\bibitem{zhou2013shifted}
------, ``Shifted subspaces tracking on sparse outlier for motion segmentation,'' in \emph{Twenty-Third International Joint Conference on Artificial Intelligence}, 2013.

\bibitem{shen2014augmented}
Y.~Shen, Z.~Wen, and Y.~Zhang, ``Augmented lagrangian alternating direction method for matrix separation based on low-rank factorization,'' \emph{Optimization Methods and Software}, vol.~29, no.~2, pp. 239--263, 2014.

\bibitem{lin2017robust}
Z.~Lin, C.~Xu, and H.~Zha, ``Robust matrix factorization by majorization minimization,'' \emph{IEEE Transactions on Pattern Analysis and Machine Intelligence}, vol.~40, no.~1, pp. 208--220, 2017.

\bibitem{wen2019robust}
F.~Wen, R.~Ying, P.~Liu, and R.~C. Qiu, ``{Robust PCA using generalized nonconvex regularization},'' \emph{IEEE Transactions on Circuits and Systems for Video Technology}, vol.~30, no.~6, pp. 1497--1510, 2019.

\bibitem{wen2019nonconvex}
F.~Wen, R.~Ying, P.~Liu, and T.-K. Truong, ``{Nonconvex regularized robust PCA using the proximal block coordinate descent algorithm},'' \emph{IEEE Transactions on Signal Processing}, vol.~67, no.~20, pp. 5402--5416, 2019.

\bibitem{quach2017non}
K.~G. Quach, C.~N. Duong, K.~Luu, and T.~D. Bui, ``{Non-convex online robust PCA: enhance sparsity via $\ell_p$-norm minimization},'' \emph{Computer Vision and Image Understanding}, vol. 158, pp. 126--140, 2017.

\bibitem{bouwmans2016handbook}
T.~Bouwmans, N.~S. Aybat, and E.-h. Zahzah, \emph{Handbook of robust low-rank and sparse matrix decomposition: Applications in image and video processing}.\hskip 1em plus 0.5em minus 0.4em\relax CRC Press, 2016.

\bibitem{wright2009robust}
J.~Wright, A.~Ganesh, S.~Rao, Y.~Peng, and Y.~Ma, ``Robust principal component analysis: Exact recovery of corrupted low-rank matrices via convex optimization,'' \emph{Advances in Neural Information Processing Systems}, vol.~22, 2009.

\bibitem{wright2010dense}
J.~Wright and Y.~Ma, ``Dense error correction via $\ell_1$-minimization,'' \emph{IEEE Transactions on Information Theory}, vol.~56, no.~7, pp. 3540--3560, 2010.

\bibitem{zhou2013greedy}
T.~Zhou and D.~Tao, ``Greedy bilateral sketch, completion \& smoothing,'' in \emph{Artificial Intelligence and Statistics}.\hskip 1em plus 0.5em minus 0.4em\relax PMLR, 2013, pp. 650--658.

\bibitem{wang2023robustHQF}
Z.-Y. Wang, X.~P. Li, H.~C. So, and Z.~Liu, ``{Robust PCA via non-convex half-quadratic regularization},'' \emph{Signal Processing}, vol. 204, p. 108816, 2023.

\bibitem{marjanovic2012l_q}
G.~Marjanovic and V.~Solo, ``On $ l\_q $ optimization and matrix completion,'' \emph{IEEE Transactions on Signal Processing}, vol.~60, no.~11, pp. 5714--5724, 2012.

\bibitem{Zhang2010NearlyUV}
C.-H. Zhang, ``Nearly unbiased variable selection under minimax concave penalty,'' \emph{Annals of Statistics}, vol.~38, pp. 894--942, 2010.

\bibitem{fan2001variable}
J.~Fan and R.~Li, ``Variable selection via nonconcave penalized likelihood and its oracle properties,'' \emph{Journal of the American Statistical Association}, vol.~96, no. 456, pp. 1348--1360, 2001.

\bibitem{gao2016penalized}
X.~Gao and Y.~Fang, ``Penalized weighted least squares for outlier detection and robust regression,'' \emph{arXiv preprint arXiv:1603.07427}, 2016.

\bibitem{roy2024robust}
S.~Roy, A.~Basu, and A.~Ghosh, ``Robust principal component analysis using density power divergence,'' \emph{Journal of Machine Learning Research}, vol.~25, no. 324, pp. 1--40, 2024.

\bibitem{wang2008new}
Y.~Wang, J.~Yang, W.~Yin, and Y.~Zhang, ``A new alternating minimization algorithm for total variation image reconstruction,'' \emph{SIAM Journal on Imaging Sciences}, vol.~1, no.~3, pp. 248--272, 2008.

\bibitem{gazzola2020inner}
S.~Gazzola, M.~E. Kilmer, J.~G. Nagy, O.~Semerci, and E.~L. Miller, ``An inner-outer iterative method for edge preservation in image restoration and reconstruction,'' \emph{Inverse Problems}, vol.~36, no.~12, p. 124004, 2020.

\bibitem{bahdanau2014neural}
D.~Bahdanau, ``Neural machine translation by jointly learning to align and translate,'' \emph{arXiv preprint arXiv:1409.0473}, 2014.

\bibitem{vaswani2017attention}
A.~Vaswani, ``Attention is all you need,'' \emph{Advances in Neural Information Processing Systems}, 2017.

\bibitem{guo2017godec+}
K.~Guo, L.~Liu, X.~Xu, D.~Xu, and D.~Tao, ``Godec+: Fast and robust low-rank matrix decomposition based on maximum correntropy,'' \emph{IEEE Transactions on Neural Networks and Learning Systems}, vol.~29, no.~6, pp. 2323--2336, 2017.

\bibitem{wang2023robust}
Z.-Y. Wang, H.~C. So, and A.~M. Zoubir, ``Robust low-rank matrix recovery via hybrid ordinary-welsch function,'' \emph{IEEE Transactions on Signal Processing}, 2023.

\bibitem{wang2014cdnet}
Y.~Wang, P.-M. Jodoin, F.~Porikli, J.~Konrad, Y.~Benezeth, and P.~Ishwar, ``{CDnet 2014: An expanded change detection benchmark dataset},'' in \emph{Proceedings of the IEEE Conference on Computer Vision and Pattern Recognition orkshops}, 2014, pp. 387--394.

\bibitem{li2004statistical}
L.~Li, W.~Huang, I.~Y.-H. Gu, and Q.~Tian, ``Statistical modelling of complex backgrounds for foreground object detection,'' \emph{IEEE Transactions on image processing}, vol.~13, no.~11, pp. 1459--1472, 2004.

\end{thebibliography}

%\begin{thebibliography}{1}
\bibliographystyle{IEEEtran}

\end{document}